\documentclass{svjour3} [12pt]
\usepackage{times}

\smartqed  % flush right qed marks, e.g. at end of proof
\usepackage{graphicx}
\usepackage{fix-cm}
\usepackage{amsmath}
\usepackage{graphicx,float,wrapfig}
\usepackage{multirow}
\usepackage{natbib}
\usepackage{rotating}
\usepackage{color}

\newcommand{\w}{\mathbf{w}}
\newcommand{\x}{\mathbf{x}}
\newcommand{\Gsub}{G_{\rm Sub}}
\newcommand{\GRR}{G_{\rm RR}}

\begin{document}

\title{Direct Learning to Rank and Rerank} 
\titlerunning{On Direct Learning to Rank and Rerank}
\author{Cynthia Rudin \and Yining Wang
}

\begin{center}
\vspace*{5pt}
\Large \textbf{Direct Learning to Rank and Rerank\footnote{Longer version (supplement) for AISTATS 2018 paper}}\vspace*{5pt}\\
\vspace*{5pt}
\normalsize \textbf{Cynthia Rudin, Duke University}\\
\normalsize \textbf{Yining Wang, Carnegie Mellon University}\\
\end{center}

\normalsize
\institute{Cynthia Rudin \at
Department of Computer Science \& Department of Electrical and Computer Engineering\\
  Duke University\\
  USA\\
  \email{cynthia@cs.duke.edu}
\and
Yining Wang \at
Machine Learning Department \\
School of Computer Science \\ 
Carnegie Mellon University \\
\email{ynwang.yining@gmail.com}
}

%\maketitle

\begin{abstract}%   <- trailing '%' for backward compatibility of .sty file
Learning-to-rank techniques have proven to be extremely useful for prioritization problems, where we rank items in order of their estimated probabilities, and dedicate our limited resources to the top-ranked items. This work exposes a serious problem with the state of learning-to-rank algorithms, which is that they are based on convex proxies that lead to poor approximations. We then discuss the possibility of  ``exact" reranking algorithms based on mathematical programming. We prove that a relaxed version of the ``exact" problem has the same optimal solution, and provide an empirical analysis. 
\end{abstract}

\begin{keywords}
{learning to rank, reranking, supervised ranking, mixed-integer programming, rank statistics, discounted cumulative gain, preference learning}
\end{keywords}

\section{Introduction}
\label{intro}
We are often faced with prioritization problems -- how can we rank aircraft in order of vulnerability to failure? How can we rank patients in order of priority for treatment? When we have limited resources and need to make decisions on how to allocate them, these ranking problems become important. The quality of a ranked list is often evaluated in terms of \textit{rank statistics}. The area under the receiver operator characteristic curve ~\citep[AUC,][]{Metz78,Bradley97}, which counts pairwise comparisons, is a rank statistic, but it does not focus on the top of a ranked list, and is not a good evaluation measure if we care about prioritization problems. For prioritization problems, we would use rank statistics that focus on the top of the ranked list, such as a weighted area under the curve that focuses on the left part of the curve. Then, since we evaluate our models using these rank statistics, we should aim to optimize them out-of-sample by optimizing them in-sample. The learning-to-rank field (also called supervised ranking) is built from this fundamental idea. Learning-to-rank is a natural fit for many prioritization problems. If we are able to improve the quality of a prioritization policy by even a small amount, it can have an important practical impact. Learning-to-rank can be used to prioritize mechanical equipment for repair \cite[e.g., airplanes, as considered by][]{ozaEtAl09}, it could be useful for prioritizing maintenance on the power grid \citep{RudinEtAl12,Rudin10}, it could be used for ranking medical workers in order of likelihood that they accessed medical records inappropriately \citep[as considered by][]{MenonEtAl13}, prioritizing safety inspections or lead paint inspections in dwellings \citep{Ghani2015}, ranking companies in order of likeliness of committing tax violations \citep[see][]{KongSa13}, ranking water pipes in order of vulnerability \citep[as considered by][]{LiEtAl13}, other areas of information retrieval \citep{Xu07aboosting,CaoETAL07,Matveeva06highaccuracy,Lafferty01documentlanguage,Li_mcrank:learning} and in almost any domain where one measures the quality of results by rank statistics. 
Learning-to-rank algorithms have been used also in sentiment analysis \citep{KesslerNi09}, natural language processing \citep{Ji06, Collins03}, image retrieval \citep{JainVa11,KangEtAl11}, and reverse-engineering product quality rating systems \citep{ChangEtAl2012}. 

This work exposes a serious problem with the state of learning-to-rank algorithms, which is that they are based on convex proxies for  rank statistics, and when these convex proxies are used, computation is faster but the quality of the solution can be poor. 

We then discuss the possibility of more direct optimization of rank statistics for predictive learning-to-rank problems. In particular, we consider a strategy of ranking with a simple ranker (logistic regression for instance) which is computationally efficient, and then reranking only the candidates near the top of the ranked list with an ``exact" method. The exact method does not have the shortcoming that we discussed earlier for convex proxies. 

For most ranking applications, we care only about the top of the ranked list; thus, as long as we rerank enough items with the exact method, the re-ranked list is (for practical purposes) just as useful as a full ranked list would be (if we could compute it with the exact method, which would be computationally prohibitive).

The best known theoretical guarantee on ranking methods is obtained by directly optimizing the rank statistic of interest \citep[as shown by theoretical bounds of][for instance]{Clemencon08, RudinSc09}
hence our choice of methodology  -- mixed-integer programming (MIP) -- for reranking in this work. Our general formulation can optimize any member of a large class rank statistics using a single mixed-integer \textit{linear} program. 
%We do not require the complexity of nonlinear programming (MINLP), even for the complicated rank statistics that we work with. 
Specifically, we can handle (a generalization of) the large class of \textit{conditional linear rank statistics}, which includes the Wilcoxon-Mann Whitney U statistic, or equivalently the Area Under the ROC Curve, %~\citep{Metz78,Bradley97}, 
the Winner-Take-All statistic, the Discounted Cumulative Gain used in information retrieval~\citep{Jarvelin00}, and the Mean Reciprocal Rank. 

Exact learning-to-rank computations need to be performed carefully; we should not refrain from solving hard problems, but certain problems are harder than others. We provide two MIP formulations aimed at the same ranking problems.  The first one works no matter what the properties of the data are. The second formulation is much faster, and is theoretically shown to produce the \textit{same} quality of result as the first formulation when there are no duplicated observations.
%One formulation is much more tractable than the other but provably produces the \textit{same} quality of solution except in pathological cases where observations are duplicated. 
Note that if the observations are chosen from a continuous distribution then duplicated observations do not occur, with probability one.

One challenge in the exact learning-to-rank formulation is the way of handling ties in score. As it turns out, the original definition of conditional linear rank statistics can be used for the purpose of evaluation but not optimization. We show that a small change to the definition can be used for optimization.

This paper differs from our earlier technical report and non-archival conference paper \citep{BertsimasChRuOR38811,BerChaRud2010}, which were focused on solving full problems to optimality, and did not consider reranking or regularization; our exposition for the formulations closely follows this past work. The technique was used by \citet{ChangEtAl2012} for the purpose of reverse engineering product rankings from rating companies that do not reveal their secret rating formula.

Section~\ref{background} of this paper introduces ranking and reranking, introduces the class of conditional linear rank statistics that we work with, and provides background on some current approximate algorithms for learning-to-rank. It also provides an example to show how ranked statistics can be ``washed out" when they are approximated by convex substitutes. 
Section~\ref{background}
also discusses a major difference between approximation methods and exact methods for optimizing rank statistics, which is how to handle ties in rank. As it turns out, we cannot optimize conditional linear rank statistics without changing their definition: a tie in score needs to be counted as a mistake. 
Section~\ref{formulations} provides the two MIP formulations for ranking, and Section~\ref{sectionrelationship} contains a proof that the second formulation is sufficient to solve the ranking problem provided that no observations are duplicates of each other. Then follows an empirical discussion in Section \ref{sectionexperiments}, designed to highlight the tradeoffs in the quality of the solution outlined above. Appendix \ref{AUCopt} contains a MIP formulation for regularized AUC maximization, and Appendix \ref{Pairwiseopt} contains a MIP formulation for a general (non bipartite) ranking problem. 

%\citep[e.g.,][]{Malago09,Furtlehner10,Hsu10} or using heuristics from combinatorial optimization that exploit problem structure to address ML tasks \citep[e.g.,][]{Cevher09,Lin10}, instead of using MIO formulations to directly solve ML problems. 

The recent work most related to ours are possibly those of \citet{Ataman06} who proposed a ranking algorithm to maximize the AUC using linear programming, and \citet{Brooks10}, who uses a ramp loss and hard margin loss rather than a conventional hinge loss, making their method robust to outliers, within a mixed-integer programming framework. The work of \citet{tan2013direct} uses a non-mathematical-programming coordinate ascent approach, aiming to approximately optimize the exact ranking measures, for large scale problems. 
%In this paper we consider pairwise ranking problems, where pairs of observations are considered. 
There are also algorithms for \textit{ordinal regression}, which is a related but different learning problem \citep{Li_mcrank:learning, crammer2001pranking,herbrich1999large}, and \textit{listwise} approaches to ranking \citep{CaoETAL07,xia2008listwise,xu2007adarank,yue2007support}. 
%Most work in learning-to-rank is for the purpose of information retrieval, as opposed to maintenance prioritization or other ranking problems that this particular work would be more relevant for. 

%There is a large body of work on theory for learning-to-rank \citep[e.g.,][]{WangEtAl13}.
%,FurnkranzHu03}.

\section{Learning-to-Rank and Learning-To-Rerank}
\label{background}
%In this section we discuss reranking, rank statistics, baseline methods for ranking in the literature, and discuss the need for handling ties in exact ranking methods.

%\subsection{Notation}
We first introduce \textit{learning-to-rank}, or \textit{supervised bipartite ranking}. The training data are labeled observations $\{(\x_i,y_i)\}_{i=1}^n$, with observations $\x_i\in X\subset\mathcal{R}^d$ and labels $y_i\in\{0,1\}$ for all $i$. The observations labeled ``1" are called ``positive observations," and the observations labeled ``0" are ``negative observations." There are $n_+$ positive observations and $n_-$ negative observations, with index sets $S_+=\{i:y_i=1\}$ and $S_-=\{k:y_k=0\}$. A ranking algorithm uses the training data to produce a scoring function $f:X\rightarrow \mathcal{R}$ that assigns each observation a real-valued score. Ideally, for a set of test observations drawn from the same (unknown) distribution as the training data, $f$ should rank the observations in order of $P(y=1|\x)$, and we measure the quality of the solution using ``rank statistics," or functions of the observations relative to each other. Note that bipartite ranking and binary classification are fundamentally different, and there are many works that explain the differences \citep[e.g.,][]{ErtekinRu11}. Briefly, classification algorithms consider a statistic of the observations relative to a decision boundary ($n$ comparisons) whereas ranking algorithms consider observations relative to each other (on the order of $n^2$ comparisons for pairwise rank statistics).
 
Since the evaluation of test observations uses a chosen rank statistic, the same rank statistic (or a convexified version of it) is optimized on the training set to produce $f$. Regularization is added to help with generalization. Thus, a ranking algorithm looks like:
\[ \min_{f\in\mathcal{F}}\;\; \textrm{RankStatistic}(f,\{\x_i,y_i\}_i) + C\cdot \textrm{Regularizer}(f).\]
This is the form of algorithm we consider for the reranking step.

\subsection{Reranking}
We are considering \textit{reranking} methods, which have two ranking steps. In the first ranking step, a base algorithm is run over the training set, and a scoring function $f_{\textrm{initial}}$ is produced and observations are rank-ordered by the score. A threshold is chosen, and all observations with scores above the threshold are reranked by another ranking algorithm which produces another scoring function $f$. To evaluate the quality of the solution on the test set, each test observation is evaluated first by $f_{\textrm{initial}}$. For the observations with scores above the threshold, they are reranked according to $f$. The full ranking of test observations is produced by appending the test observations scored by $f$ to the test observations scored only by $f_{\textrm{initial}}$.

%There have been many works on fast algorithms \citep, but not as many works on more direct methods \citep.

%Ideally for either ranking or reranking we would want to directly optimize the rank statistic. This yields the tightest guarantee on test error provided by statistical learning theoretic bounds \cite. There is work on asymptotic consistency of ranking algorithms showing that convexified losses yield consistent results with a fast type of convergence rate, however these results are not relevant here; for instance, they would need to indicate how the hypothesis space of functions changes with the size of the training set, and the convergence rate they discuss is not the same rate involved in practice. The tightest bounds we have for the finite sample case involve the exact rank statistic. 

\subsection{Rank Statistics}
We will extend the definition of conditional linear rank statistics~\citep{Clemencon08,Clemencon08_2} to include various definitions of rank. For now, we assume that there are no ties in score for any pair of observations, but we will heavily discuss ties later, and extend this definition to include rank definitions when there are ties. For the purpose of this section, the rank is currently defined so that the top of the list has the highest ranks, and all ranks are unique. The rank of an observation is the number of observations with scores at or beneath it: \[\textrm{Rank}(f(\x_i))=\sum_{t=1}^n \mathbf{1}_{[f(\x_t)\leq f(\x_i)]}.\]
%\begin{definition}
%\label{rrf}
Thus, ranks can range from 1 at the bottom to $n$ at the top. 
A \textbf{conditional linear rank statistic} (CLRS) created from scoring function $f:X\rightarrow \mathcal{R}$ is of the form
\[
%\hat{W}_n
\textrm{CLRS}(f) = \sum_{i=1}^n \mathbf{1}_{y_i=1}\phi(\textrm{Rank}(f(\x_i)).
\]
Here $\phi$ is a non-decreasing function producing only non-negative values. 
Without loss of generality, we define $a_\ell:=\phi(\ell)$, the contribution to the score if the observation with rank $\ell$ has label +1. By properties of $\phi$, we know $0\leq a_1\leq a_2\leq\cdots\leq a_n$. Then 
\begin{equation}
\label{rrf_equation}
\mathrm{CLRS}(f) = \sum_{i=1}^ny_i\sum_{\ell=1}^{n} \mathbf{1}_{[\mathrm{Rank}(f(\x_i))=\ell]}\cdot a_\ell.
\end{equation}
%\end{definition}
This class captures a broad collection of rank statistics, including the following well-known rank statistics:
\begin{itemize}\setlength{\itemsep}{0pt}
\item $a_\ell=\ell$: Wilcoxon Rank Sum (WRS) statistic, which is an affine function of the Area Under the Receiver Operator Characteristic Curve (AUC) when there are no ties in rank (that is, $f$ such that $f(\x_i)\neq f(\x_k)$ $\forall i\neq k$). 
%differs from the WMW U statistic by a constant if there are no ties. In general, we have
\begin{eqnarray*}
\mathrm{WRS}(f) &=& \sum_{i\in S_+} \textrm{Rank}(f(\x_i))
%\\&=&
=  n_+n_-\cdot \mathrm{AUC}(f) + \frac{n_+(n_++1)}{2}. %- \left(\text{\# ties within positives}\right).
\end{eqnarray*}
If ties are present, we would subtract the number of ties within the positive class from the right side of the equation above. The AUC is the fraction of correctly ranked positive-negative pairs:
\begin{equation*}
\mathrm{AUC}(f) = \frac{1}{n_+n_-}\sum_{i\in S_+}\sum_{k\in S_-}\mathbf{1}_{[f(\x_k)<f(\x_i)]}.
\end{equation*}
The AUC, when multiplied by constant $n_+n_-$, is the Mann-Whitney U statistic.
The AUC has an affine relationship with the pairwise misranking error (the fraction of positive-negative pairs in which a positive is ranked at or below a negative):
\begin{equation}\label{PME}
\textrm{PairwiseMisrankingError}(f) = 1 - \mathrm{AUC}(f) = \frac{1}{n_+n_-}\sum_{i\in S_+}\sum_{k\in S_-}\mathbf{1}_{[f(\x_k)\geq f(\x_i)]}.
\end{equation}
 Some ranking algorithms are designed to approximately minimize the pairwise misranking error, e.g., RankBoost \citep{Freund03}.
%Maximizing WRS is equivalent to maximizing WMW (or AUC) if we do not care about ties within the positive observations.
\item $a_\ell=\ell\cdot\mathbf{1}_{[\ell\geq \theta]}$ for predetermined threshold $\theta$: Related to the local AUC or partial AUC, which looks at the area under the leftmost part of the ROC curve only \citep{Clemencon07, Clemencon08, Dodd03}. The leftmost part of the ROC curve is the top portion of the ranked list. The top of the list is the most important in applications such as information retrieval and maintenance.
\item $a_\ell=\mathbf{1}_{[\ell=n]}$: Winner Takes All (WTA), which is 1 when the top observation in the list is positively-labeled and 0 otherwise \citep{BurgesETAL06}.
\item $a_\ell=\frac{1}{n-\ell+1}$: Mean Reciprocal Rank (MRR) \citep{BurgesETAL06}.
\item $a_\ell=\frac{1}{\log_2(n-\ell+2)}$: Discounted Cumulative Gain (DCG), which is used in information retrieval~\citep{Jarvelin00}.
\item $a_\ell=\frac{1}{\log_2(n-\ell+2)}\cdot\mathbf{1}_{[\ell\geq N]}$: DCG@N, which cuts off the DCG after the top N.~\citep[See, for instance,][]{Le10}.
\item $a_\ell=\ell^p$ for some $p>0$: Similar to the $P$-Norm Push, which uses $\ell_p$ norms to focus on the top of the list, the same way as an $\ell_p$ norm focuses on the largest elements of a vector~\citep{Rudin09}. 
\end{itemize}

Rank statistics have been studied in several theoretical papers \citep[e.g.,][]{pmlr-v30-Wang13}.

\subsection{Some Known Methods for Learning-To-Rank}
\label{approx_methods}
Current methods for learning-to-rank optimize convex proxies for the rank statistics discussed above. 
RankBoost \citep{Freund03} uses the exponential loss function as an upper bound for the 0-1 loss within the misranking error, 
$\mathbf{1}_{[z\leq 0]} \leq e^{-z}$,
and minimizes 
\begin{equation}\label{rb_loss}
\sum_{i\in S_+}\sum_{k\in S_-} e^{-\left(f(\x_i)-f(\x_k)\right)},
\end{equation}
whereas support vector machine ranking algorithms \citep[e.g.,][]{Joachims2002c, Herbrich00, ShenJo03} use the hinge loss $\max\{0,1-z\}$, that is:
\begin{equation}\label{svm_loss}
\sum_{i\in S_+}\sum_{k\in S_-} \max\{0,1-\left(f(\x_i)-f(\x_k)\right)\}+C\|f\|_K^2,
\end{equation}
where the regularization term is a reproducing kernel Hilbert space norm.
Other ranking algorithms include RankProp and RankNet~\citep{Caruana96, Burges05}.

\textcolor{black}{We note that the class of CLRS includes a very wide range of rank statistics, some of which concentrate on the top of the list (e.g., DCG) and some that do not (e.g.,WRS), and it is not clear which conditional linear rank statistics (if any) from the CLRS are close to the convexified loss functions (\ref{rb_loss}) and (\ref{svm_loss}). }
%It is not necessarily true that the minimizer using (\ref{rb_loss}) and (\ref{svm_loss}) for all pairs of observations will be close in practice to the minimizer of any one of the rank statistics we listed above. 
%For instance, the contribution to the loss function from each pair of observations can substantially differ between the convexified loss and the 0-1 loss, which can undermine the nuances desired from the particular rank statistic of interest.

Since the convexified loss functions do not necessarily represent the rank statistics of interest, it is not even necessarily true that an algorithm for ranking will perform better for ranking than an algorithm designed for classification; in fact, AdaBoost and RankBoost provably perform equally well for ranking under fairly general circumstances \citep{RudinSc09}. \citet{ErtekinRu11} provide a discussion and comparison of classification versus ranking methods. Ranking algorithms ultimately aim to put the observations in order of $P(y=1|x)$, and so do some classification algorithms such as logistic regression. Thus, one might consider using logistic regression for ranking \citep[e.g.,][]{Cooper94,Fine97,Perlich03}. Logistic regression minimizes:
\begin{equation}\label{lr_loss}
\sum_{i=1}^n \ln\left(1+e^{-y_if(\x_i)}\right).
\end{equation}
This loss function does not closely resemble the AUC. On the other hand, it is surprising how common it is within the literature to use logistic regression to produce a predictive model, and yet evaluate the quality of the learned model using AUC.
% To use logistic regression for ranking, we would use the probability estimates to rank the observationsAlso see
%Another algorithm that produces density estimates is AdaBoost \citep{SchapireFr11}. AdaBoost is closely related to RankBoost, in fact these algorithms are essentially equivalent in the case where $f$ is a linear combination of binary features \citep{RudinSc09}.

%Because these methods all use convex loss functions, the rank statistic of interest tends to get ``washed away" by the convexification, in the sense that the differences between rank statistics tend to get somewhat ignored by using a convex surrogate. This is precisely why we consider the possibility of reranking, where we minimize the exact rank statistic of interest. The method will still be able to run on a large scale because of the first ranking step, and the reranking step may give us the benefit of more precise ranking at the top of the list. This would be beneficial for rank statistics that focus on the top portion of the list. 

Since RankBoost, RankProp, RankNet, etc., do not directly optimize any CLRS, they do not have the problem with ties in score that we will find when we directly try to optimize a CLRS. 

%%%%%%%%%%
\subsection{Why Learning-To-Rank Methods Can Fail\label{subsec:why}}
We prove that the exponential loss and other common loss functions may yield poor results for some rank statistics.

%The fundamental premise of learning-to-rank is that better test performance can be achieved by optimizing the performance measure (a rank statistic) on the training set. This means that one should choose to optimize differently for each rank statistic. However, in practice when the same convex substitute is used to approximate a variety of rank statistics, it directly undermines this fundamental premise, and could compromise the quality of the solution. 
%If one uses the same convex proxy to optimize many different rank statistics, one does not really believe that the optimal solutions for these different rank statistics will be different.}

\begin{theorem} There is a simple one-dimensional dataset for which there exist two ranked lists (called Solution 1 and Solution 2) that are completely reversed from each other (the top of one list is the bottom of the other and vice versa) such that the WRS (the AUC), partial AUC@100, DCG, MRR and hinge loss prefer Solution 1, whereas the DCG@100, partialAUC@10 and exponential loss all prefer Solution 2. 
\end{theorem}

The proof is by construction. Along the single dimension $x$, the  dataset has 10 negatives near $x$=3, then 3000 positives near $x$=1, then 3000 negatives near $x$=0, and 80 positives near $x$=$-10$. We generated each of the four clumps of points wth a a standard deviation of 0.05 just so that there would not be ties in score. Figure \ref{Fig:distribution} shows data drawn from the distribution, where for display purposes we spread the points along the horizontal axis, but the vertical axis is the only one that matters: one ranked list goes from top to bottom (Solution 1) and the other goes from bottom to top (Solution 2).

The bigger clumps are designed to dominate rank statistics that do not decay (or decay slowly) down the list, like the WRS. The smaller clumps are designed to dominate rank statistics that concentrate on the top of the list, like the partial WRS or partial DCG.

\begin{figure}
\centering
\includegraphics[height=1.5in]{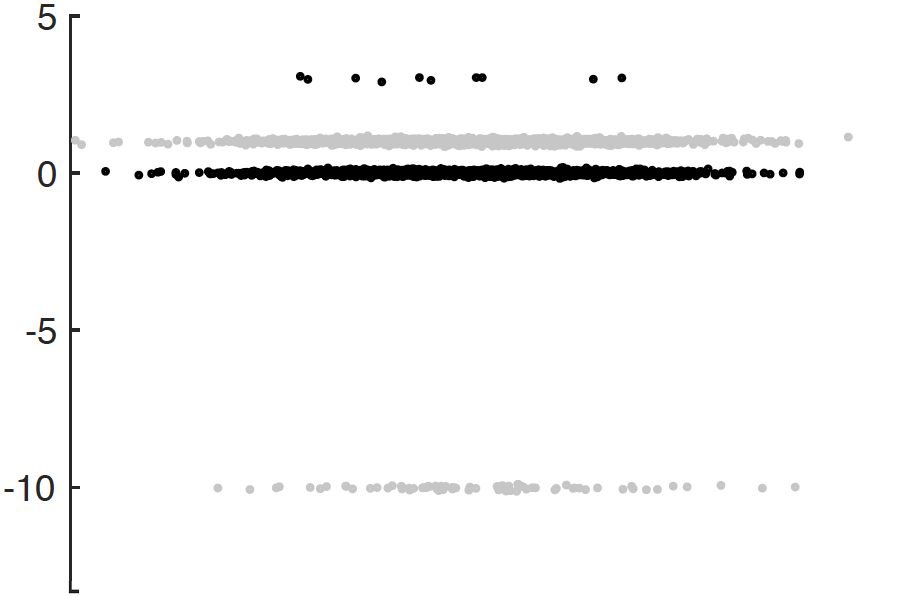}
\caption{An illustrative distribution of data. Positive observations are gray and negative observations are black. Solution 1 ranks observations from top to bottom, and Solution 2 ranks solutions from bottom to top.\label{Fig:distribution}}
\end{figure}

This theorem means that using the exponential loss to approximate the AUC, as RankBoost does, could give the completely opposite result than desired. 
It also means that using the hinge loss to approximate the partial DCG or partial AUC could yield completely the wrong result. Further, the fact that the exponential loss and hinge loss behave differently also suggests that convex losses can behave quite differently than the underlying rank statistics that they are meant to approximate. Another way to say this is that the convexification ``washes out" the differences between rank statistics. If we were directly to optimize the rank statistic of interest, the problem discussed above would vanish.

%We will consider whether various rank statistics and convexified rank statistics have similar values. This means we should see a large difference between the rank statistics for the two opposite pointing vectors, and one of these two vectors should, at the very least, always be favored over the other one when considering the various rank statistics. However, this is not the case: for some of the rank statistics, for $a_{\ell}=1$ where we are calculating the WRS, vector 1's solution yields over 13 million, whereas for vector 2's solution it is only around 6 million, and the reverse is true for calculating the 

It is not surprising that rank statistics can behave quite differently on the same dataset. Rank statistics are very different than classification statistics. Rank statistics consider every pair of observations relative to each other, so even small changes in a scoring function $f$ can lead to large changes in a rank statistic. Classification is different -- observations are considered relative only to a decision boundary.

The example considered in this section also illustrates why arguments about consistency (or lack thereof) of ranking methods \citep[e.g.,][]{KotlowskiDeHu11} are not generally relevant for practice. Sometimes these arguments rely on incorrect assumptions about the class of models used for ranking with respect to the underlying distribution of the data. These arguments also depend on how the modeler is assumed to ``change" this class as the sample size increases to infinity.
The tightest bounds available for limited function classes and for finite data are those from statistical learning theory. Those bounds support optimizing rank statistics. 

To optimize rank statistics, there is a need for more refined models; however, this refinement comes at a computational cost of solving a harder problem. This thought has been considered in several previous works on learning-to-rank \citep{Le10,ErtekinRu11,tan2013direct,Chak08,QinEtAl10}. 

%%%%%%%%%%

\subsection{{Most Learning-To-Rank Methods Have The Problem Discussed Above}}
The class of CLRS includes a very wide range of rank statistics, some of which concentrate on the top of the list (e.g., DCG) and some that do not (e.g.,WRS), and it is not clear which conditional linear rank statistics (if any) from the CLRS are close to the convexified loss functions of the ranking algorithms. 
RankBoost is not the only algorithm where problems can occur, and they can also occur for support vector machine ranking algorithms \citep[e.g.,][]{Joachims2002c, Herbrich00}
% use the hinge loss $\max\{0,1-z\}$, that is:
%\begin{equation}\label{svm_loss}
%$\sum_{i\in S_+}\sum_{k\in S_-} \max\{0,1-\left(f(\x_i)-f(\x_k)\right)\}+C\|f\|_K^2,$
%\end{equation}
%where the regularization term is a reproducing kernel Hilbert space norm.
%Other ranking algorithms include
and algorithms like RankProp and RankNet \citep{Caruana96, Burges05}.
The methods of \citet{Ataman06}, \citet{Brooks10}, and \citet{tan2013direct} have used linear relaxations or greedy methods for learning to rank, rather than exact reranking, which will have similar issues; if one optimizes the wrong rank statistic, one may not achieve the correct answer.
%It is not necessarily true that the minimizer using (\ref{rb_loss}) and (\ref{svm_loss}) for all pairs of observations will be close in practice to the minimizer of any one of the rank statistics we listed above. 
%For instance, the contribution to the loss function from each pair of observations can substantially differ between the convexified loss and the 0-1 loss, which can undermine the nuances desired from the particular rank statistic of interest.
%Since the convexified loss functions do not necessarily represent the rank statistics of interest, it is not even necessarily true that an algorithm for ranking will perform better for ranking than an algorithm designed for classification, as we discussed for AdaBoost and RankBoost. Ranking algorithms ultimately aim to put the observations in order of $P(y=1|x)$, and so do some classification algorithms such as logistic regression. Thus, 
Logistic regression is commonly used for ranking. Logistic regression minimizes:
%\begin{equation}\label{lr_loss}
$\sum_{i=1}^n \ln\left(1+e^{-y_if(\x_i)}\right).$
%\end{equation}
This loss function does not closely resemble AUC. On the other hand, it is surprising how common it is to use logistic regression to produce a predictive model, and yet evaluate the quality of the model using AUC.
% To use logistic regression for ranking, we would use the probability estimates to rank the observationsAlso see
%\citet{ErtekinRu11} provide a discussion of classification versus ranking methods.%Another algorithm that produces density estimates is AdaBoost \citep{SchapireFr11}. AdaBoost is closely related to RankBoost, in fact these algorithms are essentially equivalent in the case where $f$ is a linear combination of binary features \citep{RudinSc09}.

%Because these methods all use convex loss functions, the rank statistic of interest tends to get ``washed away" by the convexification, in the sense that the differences between rank statistics tend to be ignored by using a convex surrogate. This is why we consider the possibility of reranking, where we minimize the exact rank statistic of interest. The method will still be able to run on a large scale because of the first ranking step, and the reranking step may give us the benefit of more precise ranking at the top of the list. This could be particularly beneficial for rank statistics that focus on the top portion of the list. 

%Since RankBoost, RankProp, RankNet, etc., do not directly optimize the CLRS, they do not have a problem with ties. 
%%%%%%%%%%
%\subsection{\textcolor{black}{A Simple Illustration of Why Rank Statistics Matter, and Why Convex Proxies Can ``Wash-Out" the Differences Between Rank Statistics\label{subsec:why}}}
\textcolor{black}{
The fundamental premise of learning-to-rank is that better test performance can be achieved by optimizing the performance measure (a rank statistic) on the training set. This means that one should choose to optimize differently for each rank statistic. However, in practice when the same convex substitute is used to approximate a variety of rank statistics, it directly undermines this fundamental premise, and could compromise the quality of the solution. }
%If one uses the same convex proxy to optimize many different rank statistics, one does not really believe that the optimal solutions for these different rank statistics will be different.}
%\textcolor{black}{
%Figure \ref{Fig:distribution} shows a very simple distribution of data that serves to illustrate this point. 
%\begin{figure}
%!\includegraphics[height=3in]{RankingPaperFigure2.png}
%\caption{An illustrative distribution of data. Positive observations are red and negative observations are black. Logistic regression's unit normal is aimed upwards (vector 1), and the comparison vector is aimed downwards (vector 2).\label{Fig:distribution}}
%\end{figure}
%This distribution has two big clumps of 3000 training observations, each drawn from a normal distribution with variance 0.5, where the positive clump was generated with mean (0,1), and the negative clump was generated with mean (0,0). These bigger clumps are designed to dominate rank statistics that do not decay (or decay slowly) down the list, like the WRS. In addition, there is a smaller 10 point negative clump generated with mean (10,1) and standard deviation 0.05, and a positive clump of 200 points generated with mean -3 and standard deviation 0.05. These smaller clumps are designed to dominate rank statistics that concentrate on the top of the list, like the partial WRS or MRR. 
%}
\textcolor{black}{
If convexified rank statistics are a reasonable substitute for rank statistics, we would expect to see that (i) the rank statistics are reasonably approximated by their convexified versions, (ii) if we consider several convex proxies for the same rank statistic (in this case AUC), then they should all behave very similarly to each other, and similarly to the true (non-convexified) AUC. However, as we discussed, neither of these are true. 
%A key observation of this is the presence of ties.
%In particular:
}

\subsection{Ties and Problematic, Thus Use ResolvedRank and Subrank}
Dealing with ties in rank is critical when directly optimizing rank statistics. If a tie in rank between a positive and negative is considered as correct, then an optimal learning algorithm would produce the trivial scoring function $f(x)=\textit{constant}\;\forall x$; this solution would unfortunately attain the highest possible score when optimizing any pairwise rank statistic. This problem happens, for instance, with the definition of \citet{Clemencon08}, that is: 
\[\textrm{RankCV}(f(\x_i))=\sum_{k=1}^n \mathbf{1}_{f(\x_k)\leq f(\x_i)},\]
which counts ties in score as correct. Using this definition for rank in the CLRS:
\begin{equation}
\label{rrf_equation}
\mathrm{CLRS}_\textrm{CV}(f) = \sum_{i=1}^ny_i\sum_{\ell=1}^{n} \mathbf{1}_{[\mathrm{RankCV}(f(\x_i))=\ell]}\cdot a_\ell.
\end{equation}
we find that optimizing CLRS$_{\textrm{CV}}$ directly yields the trivial solution that all observations get the same score. So this definition of rank should not be used.

\textcolor{black}{We need to encourage our ranking algorithm not to produce ties in score, and thus in rank. To do this, we pessimistically consider a tie between and positive and a negative as a misrank. We will use two definitions of rank within the CLRS -- \textit{ResolvedRanks} and \textit{Subranks}. For ResolvedRanks, when negatives are tied with positives, we force the negatives to be higher ranked. For Subranks, we do not force this, but when we optimize the CLRS, we will prove that ties are resolved this way anyway.}

The assignment of ResolvedRanks and Subranks are not unique, there can be multiple ways to assign ResolvedRanks or Subranks for a set of observations.

\textit{We define the \textbf{Subrank} by the following formula:}
\begin{align*}
&\text{Subrank}(f(\x_i)) = \sum_{k=1}^n\mathbf{1}_{[f(\x_k)<f(\x_i)]}, \quad\forall i=1,\dots,n.
\end{align*}
The Subrank of observation $i$ is the number of observations that score strictly below it. Subranks range from 0 to $n-1$ and the CLRS becomes:
\begin{equation}
\label{rrf_equation}
\mathrm{CLRS}_\textrm{Subrank}(f) = \sum_{i=1}^ny_i\sum_{\ell=1}^{n} \mathbf{1}_{[\mathrm{Subrank}(f(\x_i))=\ell-1]}\cdot a_\ell.
\end{equation}
Observations with equal score have tied Subranks. 

\textrm{ResolvedRanks} are defined as follows, where the tied ranks are resolved pessimistically.
%\begin{definition}
%\label{rank}
\textit{\textbf{ResolvedRanks} are assigned so that:
\begin{enumerate}\setlength{\itemsep}{0pt}
\item The ResolvedRank of an observation is greater than or equal to its Subrank.
\item If a positive observation and a negative observation have the same score, then the negative observation gets a higher ResolvedRank.
\item Each possible ResolvedRank, 0 through $n-1$, is assigned to exactly one observation.
\end{enumerate}
}
%\end{definition}
The SubRanks and ResolvedRanks are equal to each other when there are no ties in score. 
We provide one possible assignment of Subranks and ResolvedRanks in Figure~\ref{rankdefinitions} to demonstrate the treatment of ties.
We then have the CLRS with ResolvedRanks as:
\begin{equation}
\label{clrsresolvedranks}
\mathrm{CLRS}_\textrm{ResolvedRank}(f) = \sum_{i=1}^ny_i\sum_{\ell=1}^{n} \mathbf{1}_{[\mathrm{ResolvedRank}(f(\x_i))=\ell-1]}\cdot a_\ell.
\end{equation}
The ResolvedRanks are the quantity of interest, as optimizing them will provide a scoring function with minimal misranks and minimal ties between positives and negatives. 

\begin{figure}
\centering
\begin{small}
\begin{tabular}{|l|c c c c c c c c c|}
\hline
\textbf{Label} $y_i$ & $+$ & $+$ & $-$ & $-$ & $-$ & $+$ & $+$ & $-$ & $+$\\
\textbf{Score} $f(\x_i)$ & 6.2 & 6.2 & 5.8 & 4.6 & 3.1 & 3.1 & 2.3 & 1.7 & 1.7\\
\hline
\textbf{SubRank} & 7 & 7 & 6 & 5 & 3 & 3 & 2 & 0 & 0\\
\textbf{ResolvedRank} & 8 & 7 & 6 & 5 & 4 & 3 & 2 & 1 & 0\\
\hline
\end{tabular}
\end{small}
\caption{Demonstration of rank definitions.}
\label{rankdefinitions}
\end{figure}

Note that ties are not fundamental in other statistical uses of rank statistics, such as hypothesis testing. Ties are usually addressed by fixing them, or assigning the tied observations a (possibly fractional) rank that is the average (e.g., tied observations would get ranks 7.5 rather than 7 and 8)~\citep[see][]{Tamhane00, Wackerly02}. Ties are not treated uniformly across statistical applications~\citep{Savage57}, and there has been comparative work on treatment of ties \citep[e.g.,][]{Putter55}. This differs from when we optimize rank statistics, where ties are of central importance as we discussed.
%we need to handle ties with care in order to avoid trivial solutions.

\section{Reranking Formulations Using ResolvedRanks and Subranks}\label{formulations}

Here we produce the two formulations -- one for optimizing the regularized CLRS with ResolvedRanks, and the other for optimizing the regularized CLRS with Subranks.

\subsection{Maximize the Regularized CLRS with ResolvedRanks}

We would like to optimize the general CLRS, for any choices of the $a_\ell$'s, where we want to penalize ties in rank between positives and negatives, and we would also like a full ranking of observations. Thus, we will directly optimize $\mathrm{CLRS}_\textrm{ResolvedRank}(f) + C\cdot \textrm{Regularizer}(f)$ for our reranking algorithm. Our hypothesis space is the space of linear scoring functions $f(\x_i)=\w^T\x_i$, where $\w\in\mathcal{R}^d$. 
\begin{align*}
\max_{\w\in \mathcal{R}^d}&\mathrm{CLRS}_\textrm{ResolvedRank}(\w) - C \|\w\|_0\\
=\max_{\w\in \mathcal{R}^d}& \sum_{i=1}^ny_i\sum_{\ell=1}^{n} \mathbf{1}_{[\mathrm{ResolvedRank}(\w^T\x_i))=\ell-1]}\cdot a_\ell - C\|\w\|_0.
\end{align*}
Nonlinearities can be incorporated as usual by including additional variables, such as indicator variables or nonlinear functions of the original variables. We optimize over choices for vector $\w$.

Building up to the formulation, we will create the binary variable $t_{i\ell}$ so that it is 1 for $\ell \leq \text{ResolvedRank}(f(\x_i))+1$ and 0 otherwise. That is, if observation $i$ has ResolvedRank equal to 5, then $t_{i1},...,t_{i6}$ are all 1 and $t_{i7},...,t_{in}$ are 0. Then 
\begin{equation}\label{tele}
\sum_{\ell=1}^n (a_\ell-a_{\ell-1})t_{i\ell}
\end{equation}
is a telescoping sum for $\ell\leq $ResolvedRank$(f(\x_i))+1$. When we define $a_0=0$, the sum (\ref{tele}) becomes simply $a_{\textrm{ResolvedRank}(f(\x_i))+1}$, or equivalently, the term from (\ref{clrsresolvedranks}):
\[\sum_{\ell=1}^{n} \mathbf{1}_{[\mathrm{ResolvedRank}(f(\x_i))=\ell-1]}\cdot a_\ell.\]
As in (\ref{clrsresolvedranks}) we multiply by $y_i$ and sum over observations to produce the CLRS$_{\textrm{ResolvedRank}}$. Doing this to (\ref{tele}), CLRS$_{\textrm{ResolvedRank}}$ becomes:
\[
\sum_{i=1}^n y_i \sum_{\ell=1}^n (a_\ell-a_{\ell-1})t_{i\ell} \quad\textrm{where } a_0=0.
\]
By definition $t_{i1}=1$ for all $i$, so we can simplify the CLRS$_{\textrm{ResolvedRank}}$  function above to:
\begin{displaymath}
\sum_{i\in S_+}\left(\sum_{\ell=2}^n (a_\ell-a_{\ell-1})t_{i\ell} + a_1\right) = \vert S_+\vert a_1 + \sum_{i\in S_+}\sum_{\ell=2}^n (a_\ell-a_{\ell-1})t_{i\ell}.
\end{displaymath}
Note that the differences $a_\ell-a_{\ell-1}$ are all nonnegative. When they are zero they cannot contribute to the CLRS$_{\textrm{ResolvedRank}}$  function. When they are strictly positive there can be a contribution made to the CLRS$_{\textrm{ResolvedRank}}$  function. Thus, we introduce notation $\tilde{a}_\ell=a_\ell-a_{\ell-1}$ and $S_{\textrm{r}} = \{\ell\geq 2:\tilde{a}>0\}$ which are used in both formulations below. The CLRS$_{\textrm{ResolvedRank}}$ becomes:
\begin{equation}\label{simplif}
 \vert S_+\vert a_1 +\sum_{i\in S_+}\sum_{\ell\in S_{\textrm{r}}} \tilde{a}_{\ell}t_{i\ell}.
\end{equation}
We will maximize this, which means that the $t_{i\ell}$'s will be set to 1 when possible, because the $\tilde{a}_{\ell}$'s in the sum are all positive. When we maximize, we do not need the constant $\vert S_+\vert a_1$ term.

We define integer variables $r_i\in[0,n-1]$ to represent the ResolvedRanks of the observations.Variables $r_i$ and $t_{i\ell}$ are related in that 
$t_{i\ell}$ can only be 1 when $\ell \leq r_i+1$, implying $t_{i\ell}\leq \frac{r_i}{\ell-1}$.

We use linear scoring functions, so the score of instance $\x_i$ is $\w^T\x_i$. Variables $z_{ik}$ are indicators of whether the score of observation $i$ is above the score of observation $k$. Thus we want to have $z_{ik}=1$ if $\w^T\x_i>\w^T\x_k$ and  $z_{ik}=0$ otherwise. Beyond this we want to ensure no ties in score, so we want all scores to be at least $\varepsilon$ apart. This will be discussed further momentarily.

Our first ranking algorithm is below, which maximizes the regularized CLRS using ResolvedRanks. 
\begin{align}
\label{rrf_rank_formulation}
%P_{\mathrm{RRF,rank}}(\varepsilon): 
\textrm{argmax}_{\w,\gamma_j,z_{ik},t_{i\ell},r_i\forall i,k,\ell,j} \quad &\sum_{i\in S_+}\sum_{\ell\in S_{\textrm{r}}} \tilde{a}_lt_{i\ell} - C\sum_j \gamma_j\;\;\;\; \text{ subject to}\\\label{rank_constraint1}
&z_{ik}\leq \w^T(\x_i-\x_k)+1-\varepsilon, \quad\forall i,k=1,\dots,n,\\\label{rank_constraint2}
&z_{ik}\geq \w^T(\x_i-\x_k), \quad\forall i,k=1,\dots,n,\\
&\gamma_j\geq w_{j}\label{congamma1rr}\\
&\gamma_j\geq -w_j \label{congamma2rr}\\
\label{rank_constraint3}
&r_i-r_k\geq 1+n(z_{ik}-1), \quad\forall i,k=1,\dots,n,\\\label{rank_constraint4}
&r_k-r_i\geq 1-nz_{ik}, \quad\forall i\in S_+, k\in S_-,\\\label{rank_constraint5}
&r_k-r_i\geq 1-nz_{ik}, \quad\forall i,k\in S_+, i<k,\\\label{rank_constraint6}
&r_k-r_i\geq 1-nz_{ik}, \quad\forall i,k\in S_-, i<k,\\\label{rank_constraint7}
&t_{i\ell}\leq \frac{r_i}{\ell-1}, \quad\forall i\in S_+, \ell\in S_{r},\\\notag
&-1\leq w_j\leq 1, \quad\forall j=1,\dots,d,\\\notag
&0\leq r_i\leq n-1, \quad\forall i=1,\dots,n,\\\notag
&z_{ik}\in \{0,1\}, \quad\forall i, k=1,\dots,n,\\\notag
&t_{i\ell}\in \{0,1\}, \quad\forall i\in S_+, \ell\in S_{r},\\
&\gamma_j \in \{0,1\}, \quad \forall j\in\{1,...d\}.
\end{align}
To ensure that solutions with ranks that are close together are not feasible, Constraint~(\ref{rank_constraint1}) forces $z_{ik}=0$ if $\w^T\x_i-\w^T\x_k<\varepsilon$, and Constraint~(\ref{rank_constraint2}) forces $z_{ik}=1$ if $\w^T\x_i-\w^T\x_k>0$. Thus, a solution where any two observations have a score difference above 0 and less than $\varepsilon$ is not feasible. (Note that these constraints alone do not prevent a score difference of exactly 0; for that we need the constraints that follow.) Constraints (\ref{congamma1rr}) and (\ref{congamma2rr}) define the $\gamma_j$'s to be indicators of nonzero coefficients $w_j$.

Constraints (\ref{rank_constraint3})-(\ref{rank_constraint6}) are the ``tie resolution" equations.
Constraint~(\ref{rank_constraint3}) says that for any pair $(\x_i,\x_k)$, if the score of $i$ is larger than that of $k$ so that  $z_{ik}=1$, then  $r_i\geq r_k+1$. That handles the assignment of ranks when there are no ties, so now we need only to resolve ties in the score. We have Constraint~(\ref{rank_constraint4}) that applies to positive-negative pairs: when the pair is tied, this constraint forces the negative observation to have higher rank. Similarly, Constraints~(\ref{rank_constraint5}) and~(\ref{rank_constraint6}) apply to positive-positive pairs and negative-negative pairs respectively, and state that ties are broken lexicographically, that is, according to their index in the dataset. 

We discussed Constraint~(\ref{rank_constraint7}) earlier, which provides the definition of $t_{i\ell}$ so that $t_{i\ell}=1$ whenever $\ell\leq r_i+1$. Also we force the $w_j$'s to be between -1 and 1 so their values do not go to infinity and so that the $\varepsilon$ values are meaningful, in that they can be considered relative to the maximum possible range of $w_j$.

\subsection{Maximize the Regularized CLRS with Subranks}
We are solving:
\begin{align*}
\max_{\w\in \mathcal{R}^d}&\mathrm{CLRS}_\textrm{Subrank}(\w) - C \|\w\|_0\\
=\max_{\w\in \mathcal{R}^d}& \sum_{i=1}^ny_i\sum_{\ell=1}^{n} \mathbf{1}_{[\mathrm{Subrank}(\w^T\x_i))=\ell-1]}\cdot a_\ell - C\|\w\|_0.
\end{align*}

Maximizing the Subrank problem is much easier, since we do not want to force a unique assignment of ranks. This means the ``tie resolution" equations are no longer present. We can directly assign a Subrank for observation $i$ by $r_i=\sum_{k=1}^n z_{ik}$ because it is exactly the count of observations ranked beneath observation $i$; that way the $r_i$ variables do not even need to appear in the formulation.

Here is the formulation:
\begin{align}
\label{rrf_minrank_formulation}
&\textrm{argmax}_{\w,\gamma_j,z_{ik},t_{i\ell} \forall i,k,\ell,j} \quad \sum_{i\in S_+}\sum_{\ell\in S_r} \tilde{a}_lt_{i\ell} - C\sum_j \gamma_j\;\;\text{subject to}\\
\label{minrank_constraint2}
&t_{i\ell}\leq \frac{1}{\ell-1}\sum_{k=1}^n z_{ik}, \quad\forall i\in S_+,\ell\in S_r,\\
\label{minrank_constraint1}
&z_{ik}\leq \w^T(\x_i-\x_k)+1-\varepsilon, \quad\forall i\in S_+,k=1,\dots,n,\\
&\gamma_j\geq w_{j}\label{congamma1}\\
&\gamma_j\geq -w_j \label{congamma2}\\
&z_{ik}+z_{ki} = \mathbf{1}_{[\x_i\neq \x_k]}, \quad\forall i,k\in S_+,\label{minrank_constraint3}\\
&t_{i\ell}\geq t_{i,\ell+1}, \quad\forall i\in S_+, \ell\in S_r\setminus\max(\ell \in S_r),\\\label{minrank_constraint5}
&\sum_{i\in S_+}\sum_{\ell\in S_r} \tilde{a}_lt_{i\ell}\leq\sum_{\ell=1}^n a_\ell,\\\label{minrank_constraint6}
&z_{ik}=0, \quad\forall i\in S_+,k=1,\dots,n, \x_i=\x_k,\\\notag
&-1\leq w_j\leq 1, \quad\forall j=1,\dots,d,\\\notag
&t_{i\ell},z_{ik},\gamma_j\in \{0,1\}, \quad\forall i\in S_+,\ell\in S_r, k=1,\dots,n, \;\;j\in\{1,...d\}.
\end{align}

Constraint (\ref{minrank_constraint2}) is similar to Constraint (\ref{rank_constraint7}) from the ResolvedRank formulation. Since we are maximizing with respect to the $t_{i\ell}$'s, the $z_{ik}$'s will naturally be maximized by Constraint (\ref{minrank_constraint2}). 
Thus we need to again force the $z_{ik}$'s down to 0 when $\w^T\x_i-\w^T\x_k<\varepsilon$, which is done via Constraint  (\ref{minrank_constraint1}). Constraints (\ref{congamma1}) and (\ref{congamma2}) define the $\gamma_j$'s to be indicators of nonzero coefficients $w_j$.
It is not necessary to include Constraints~(\ref{minrank_constraint3}) through~(\ref{minrank_constraint6}); they are there only to speed up computation, by helping to make the linear relaxation of the integer program closer to the set of feasible integer points. For the experiments in this paper they did not substantially speed up computation and we chose not to use them.

%The minimum rank is $\sum_{k=1}^n z_{ik}$. Constraints~(\ref{minrank_constraint1}) and~(\ref{minrank_constraint2}) are similar to constraints~(\ref{rank_constraint1}) and~(\ref{rank_constraint7}) from $P_{\mathrm{RRF,rank}}$. 

Beyond the formulations presented here, we have placed a formulation for optimizing the regularized AUC in the Appendix \ref{AUCopt} and another formulation for optimizing the general pairwise rank statistic that inspired RankBoost \citep{Freund03} in Appendix \ref{Pairwiseopt}.

%%%%%%%%%%%%%%%%%%%%%%%%%%%%%%%%%%%%%
\section{Why Subranks Are Often Sufficient}
\label{sectionrelationship}

The ResolvedRank formulation above has $2d+n^2+n_+|S_r|+n$ variables, which is the total number of $w$, $\gamma$, $z$, $t$, and $r$ variables. The Subrank formulation on the other hand has only $2d+n_+n+n_+|S_r|$ variables, since we only have $w$, $\gamma$, $z$, and $t$. This difference of $n_-\cdot n+n$ variables can heavily influence the speed at which we are able to find a solution. We would ultimately like to get away with solving the Subrank problem rather than the ResolvedRank problem. This would allow us to scale up our reranking problem substantially. In this section we will show why this is generally possible.

Denote the objectives as follows, where we have $f(\x_i)=\w^T\x_i$.
\begin{eqnarray*}
\GRR(f)&:=&\sum_{i=1}^ny_i\sum_{\ell=1}^{n} \mathbf{1}_{[\mathrm{ResolvedRank}(f(\x_i))=\ell-1]}\cdot a_\ell - C\|\w\|_0\\
\Gsub(f)&:=&\sum_{i=1}^ny_i\sum_{\ell=1}^{n} \mathbf{1}_{[\mathrm{Subrank}(f(\x_i))=\ell-1]}\cdot a_\ell - C\|\w\|_0.
\end{eqnarray*}
In this section, we will ultimately prove that any maximizer of $\Gsub$ also maximizes $\GRR$. This is true under a very general condition, which is that there are no exactly duplicated observations. The reason for this condition is not completely obvious. In the Subrank formulation, if two observations are exactly the same, they will always get the same score and Subrank - there is no mechanism to resolve ties and assign ranks. This causes problems when approximating the ResolvedRank with the Subrank. We remark however, that this should not be a problem in practice. First, we can check in advance whether any of our observations are exact copies of each other, so we know whether it is likely to be a problem. Second, if we do have duplicated observations, we can always slightly perturb the $\x$ values of the duplicated observations so they are not identical. Third, we remark that if the data are chosen from a continuous distribution, with probability 1 the observations will all be distinct anyway. We have found that in practice the Subrank formulation does not have problems even when there are ties.

In the first part of the section, we consider whether there are maximizers of $\GRR$ that have no ties in score, in other words, solutions $\w$ where $f(\x_i)\neq f(\x_k)$ for any two observations $i$ and $k$. Assuming such solutions exist, we then show that any maximizer of $\Gsub$ is also a maximizer of $\GRR$. This is the result within Theorem \ref{step1}. In the second part of the section, we show that the assumption we made for Theorem \ref{step1} is always satisfied, assuming no duplicated observations. That is, a maximizer of $\GRR$ with no ties in score exists.
The outline within our technical report \citep{BertsimasChRuOR38811} follows a similar outline but does not include regularization.

The following lemma establishes basic facts about the two objectives:
%Here is the first step of the proof:
\begin{lemma}
\label{equalobj}
$\Gsub(f) \leq \GRR(f)$ for all $f$. Further, $\Gsub(f) = \GRR(f)$ for all $f$ with no ties.
%\begin{enumerate}\setlength{\itemsep}{0pt}
%\item[a.] For any $f$, $P_{\mathrm{minrank}}(f)\leq P_{\mathrm{rank}}(f)$.
%\item[b.] For any $f$ such that there are no ties in score, $P_{\mathrm{minrank}}(f)=P_{\mathrm{rank}}(f)$.
%\end{enumerate}
\end{lemma}
\begin{proof}
Choose any function $f$. Since by definition Subrank$(f(x_i))$$\leq$ ResolvedRank$(f(x_i))$ $\forall i$, and since the $a_{\ell}$ are nondecreasing, 
%\begin{enumerate}
%\item[a.] The first part of Definition~\ref{rank} says that $\mathrm{minrank}_f(x_i)\leq\mathrm{rank}_f(x_i)$ for all $i=1,\dots,n$. Since the $a_\ell$ are non-decreasing with $\ell$,
\begin{align}
\label{minrank_rank}
\sum_{\ell=1}^{n} \mathbf{1}_{[\mathrm{Subrank}(f(x_i))=\ell-1]}\cdot a_\ell &= a_{(\mathrm{Subrank}(f(x_i))+1)}\\\notag
&\leq a_{(\mathrm{ResolvedRank}(f(x_i))+1)} \\\notag
&= \sum_{\ell=1}^{n} \mathbf{1}_{[\mathrm{ResolvedRank}(f(x_i))=\ell-1]}\cdot a_\ell\quad\forall i.
\end{align}
Multiplying both sides by $y_i$, summing over $i$ and subtracting the regularization term from both sides yields $\Gsub(f) \leq \GRR(f)$. When no ties are present (that is, $f(x_i)\neq f(x_k)$ $\forall i\neq k$), Subranks and ResolvedRanks are equal, and the inequality above becomes an equality, and in that case, $\Gsub(f) = \GRR(f)$. 
%Combining this result with~(\ref{p_rank}) and~(\ref{p_minrank}), we have $P_{\mathrm{minrank}}(f)\leq P_{\mathrm{rank}}(f)$.
%\item[b.] 
%If there are no ties in score, then it is clear from Definition~\ref{rank} that we have $\mathrm{minrank}_f(x_i)=\mathrm{rank}_f(x_i)$ for all $i$. Thus the inequality in~(\ref{minrank_rank}) becomes an equality, and $P_{\mathrm{minrank}}(f)=P_{\mathrm{rank}}(f)$.
%\end{enumerate}
\end{proof}
This lemma will be used within the following theorem which says that maximizers of $\Gsub$ are maximizers of $\GRR$.
\begin{theorem}
\label{step1}
Assume that the set $\mathrm{argmax}_f\GRR(f)$ contains at least one function $\bar{f}$ having no ties in score. Then any $f^*$ such that
%Let $\bar f\in\mathrm{argmax}_fP_{\mathrm{rank}}(f)$ such that there are no ties in score, that is, $\bar f(x_i)\neq\bar f(x_k)$ for all $i\neq k$. If 
$f^*\in \mathrm{argmax}_f\Gsub(f)$ also obeys
%\begin{equation*}
$f^*\in \mathrm{argmax}_f\GRR(f)$.
%\end{equation*}
\end{theorem}
\begin{proof}
Assume there exists $\bar f\in\mathrm{argmax}_f\GRR(f)$ such that there are no ties in score. %By Lemma~\ref{optexist}, there exists a maximizer for $P_{\mathrm{minrank}}(f)$.
Since $\bar f$ is a maximizer of $\GRR$ and does not have ties, it is also a maximizer of $\Gsub$ by Lemma \ref{equalobj}:
\[
\Gsub(\bar f) =\GRR(\bar f)=\max_f \GRR(f) \geq \max_f \Gsub(f), \textrm{ thus } \Gsub(\bar f)=\max_f \Gsub(f).
\]
Let $f^*$ be an arbitrary maximizer of $\Gsub(f)$ (not necessarily tie-free). We claim that $f^\star$ is also a maximizer of $\GRR$. Otherwise, 
\[
\GRR(f^\star ) < \GRR(\bar f) \overset{(a)}{=} 
 \Gsub(\bar f) \overset{(b)}{=}
 \Gsub(f^\star)
 \overset{(c)}{\leq} \GRR(f^{\star}),
 \]
which is a contradiction. Equation (a) comes from Lemma \ref{equalobj} applied to $\bar f$. Equation (b) comes from the fact that both $\bar f$ and $f^*$ are maximizers of $\Gsub$. Inequality {(c)} comes from Lemma \ref{equalobj} applied to $f^*$.

%which implies $P_{\mathrm{minrank}}(f^*)\geq P_{\mathrm{minrank}}(\bar f)$.  We know $P_{\mathrm{minrank}}(\bar f)=P_{\mathrm{rank}}(\bar f)$ by Lemma~\ref{equalobj}b.
%Suppose $f^*$ does not maximize $P_{\mathrm{rank}}(f)$, so $P_{\mathrm{rank}}(\bar f)>P_{\mathrm{rank}}(f^*)$. Then
%\begin{equation*}
%P_{\mathrm{minrank}}(f^*)\geq P_{\mathrm{minrank}}(\bar f) = P_{\mathrm{rank}}(\bar f) > P_{\mathrm{rank}}(f^*).
%\end{equation*}
%This contradicts Lemma~\ref{equalobj}a, so $f^*\in\mathrm{argmax}_fP_{\mathrm{rank}}(f)$.
\end{proof}

Interestingly enough, it is true that if $\bar{f}$ maximizes $\GRR(f)$ and it has no ties, then $\bar{f}$ also maximizes $\Gsub(f)$. In particular, 
\begin{equation*}
\max_f\Gsub(f)\leq \max_f\GRR(f) \leq \GRR(\bar f)=\Gsub(\bar f).
\end{equation*}

%It is interesting to note that under the condition of Theorem~\ref{step1}, namely that $\bar f$ maximizes $P_{\mathrm{rank}}(f)$ without any ties in score, we can also show $\bar f\in\mathrm{argmax}_fP_{\mathrm{minrank}}(f)$. By both parts of Lemma~\ref{equalobj}, we have that for any $f$,
%\begin{equation*}
%P_{\mathrm{minrank}}(f)\leq P_{\mathrm{rank}}(f)\leq P_{\mathrm{rank}}(\bar f)=P_{\mathrm{minrank}}(\bar f).
%\end{equation*}
%Thus any $f$ that maximizes $P_{\mathrm{rank}}(f)$ without any ties in score maximizes $P_{\mathrm{minrank}}(f)$ as well.

Note that so far, the results about $\GRR$ and $\Gsub$ hold for functions from any arbitrary set; we did not need to have $f=\w^T\x$ in the preceding computations. In what follows we take advantage of the fact that $f$ is a linear combination of features in order to perturb the function away from ties in score. With this method we will be able to achieve the same maximal value of $\GRR$ but with no ties.  

Define $M$ to be the maximum absolute value of the features, so that for all $i,j$, we have $|x_{ij}|\leq M$. 

\begin{lemma}
\label{LemmaHighDim}
If we are given $\bar f\in\mathrm{argmax}_f\GRR(f)$ that yields a scoring function $\bar f(\x)=\bar \w^T\x$ with ties, it is possible to construct a perturbed scoring function $\hat{f}$ that: % \in\mathrm{argmax}_f\GRR(f)$ that:
\begin{description}
\item[i] preserves all pairwise orderings, 
$\bar f(\x_i)>\bar f(\x_k)\Rightarrow\hat f(\x_i)>\hat f(\x_k)$, % and $\hat f(\x_i)\neq\hat f(\x_k)$ for all $i\neq k$,  
\item[ii] has no ties, $\hat f(\x_i)\neq \hat f(\x_k)$ for all $i, k$.
%\item[iii] preserves values of $\GRR$, so that $\GRR(\bar f)=\GRR(\hat f)$.
\item[iii] has $\|\bar{\w}\|_0=\|\hat{\w}\|_0$.
\end{description}
This result holds whenever no observations are duplicates of each other, $\x_i\neq \x_k\;\; \forall i,k$.
\end{lemma}

\begin{proof}
We will construct $\hat f(\x)=\hat{\w}^T\x$ using the following procedure:
\begin{description}
\item[Step 1] Find the nonzero indices of $\bar \w$: let $\bar J:=\{j:\bar{w}_j\neq 0\}$.  Choose a unit vector $\mathbf{v}$ in $\mathcal{R}^{|J|}$ uniformly at random. Construct vector $\mathbf{u}\in\mathcal{R}^d$ to be equivalent to $\mathbf{v}$ for $\mathbf{u}$ restricted to the dimensions $J$ and 0 otherwise. %Thus, $\|\mathbf{u}\|_2=1$.
\item[Step 2] Choose real number $\delta$ to be between 0 and $\eta$, where
\[ 
\eta = \min\left\{\frac{\mathrm{margin}_{\bar \w}}{2M\sqrt{d}}, \min_{j\in\bar{J}} |w_j|\right\}
\]
where in the above expression
\begin{equation*}
\mathrm{margin}_{\bar \w} = \min_{\{i,k: \bar f(\x_i)>\bar f(\x_k)\}} \left(\bar f(\x_i) - \bar f(\x_k)\right).
\end{equation*}
\item[Step 3] Construct $\hat \w$ as follows: $\hat \w = \bar \w + \delta \mathbf{u}$.
\end{description}
With probability one, we will show that $\hat f(\x)=\hat \w^T\x$ preserves pairwise orderings of $\bar f$ but with no ties. 

We will prove each part of the lemma separately.

\noindent\textbf{Proof of (i)} 
%We show $\hat f$ preserves existing pairwise orderings. That is, we show that if $f(\x_i)> f(\x_k)$ the $\hat f(\x_i)<\hat f(\x_k)$. 
We choose any two observations $\x_i$ and $\x_k$ where  $\bar f(\x_i)>\bar f(\x_k)$, and we need to show that $\hat f(\x_i)>\hat f(\x_k)$.  
%Consider any pairwise ordering by choosing two observations $x_1$ and $x_2$ such that $\bar f(x_1)>\bar f(x_2)$. Now
\begin{align}
\nonumber
\hat f(\x_i)-\hat f(\x_k) &= (\bar \w+\delta \mathbf{u})^T(\x_i-\x_k)
%\\\nonumber
%\sum_{j=1}^d (w_j+\delta u_j)(x_{1j}-x_{2j})\\
= \bar \w^T(\x_i-\x_k)+\delta \mathbf{u}^T(\x_i-\x_k)\\%\sum_{j=1}^d w_j(x_{1j}-x_{2j}) + \delta\sum_{j=1}^d u_j(x_{1j}-x_{2j})\\
&= \bar f(\x_i)-\bar f(\x_k)+\delta \mathbf{u}^T(\x_i-\x_k)\label{fdiff}%+ \delta\sum_{j=1}^d u_j(x_{1j}-x_{2j})\\
\geq \mathrm{margin}_{\bar \w}+\delta \mathbf{u}^T(\x_i-\x_k).% + \delta\sum_{j=1}^d u_j(x_{1j}-x_{2j}).
\end{align}
In order to bound the right hand side away from zero we will use that:
\begin{equation}
\label{x1x2}
\|\x_i-\x_k\|_2 = \left(\sum_{j=1}^d (x_{ij}-x_{kj})^2\right)^{1/2}\leq \left(\sum_{j=1}^d (2M)^2\right)^{1/2} = 2M\sqrt{d}.
\end{equation}
Now,
\begin{equation*}
\left|\delta \mathbf{u}^T(\x_i-\x_k)\right|\overset{(a)}{\leq} \delta \|\mathbf{u}\|_2\|\mathbf{x}_i-\mathbf{x}_k\|_2\overset{(b)}{\leq}  
\delta\cdot 2M\sqrt{d} \overset{(c)}{<}  \frac{\mathrm{margin}_{\bar \w}}{2M\sqrt{d}}\cdot 2M\sqrt{d} = \mathrm{margin}_{\bar \w}.
\end{equation*}
Here, inequality (a) follows from the Cauchy-Schwarz inequality, (b) follows from (\ref{x1x2}) and that $\|\mathbf{u}\|_2=1$, and (c) follows from the bound on $\delta$ from Step 2 of the procedure for constructing $\hat{f}$ above. Thus
%\begin{equation}
%\label{mindisteqn}
$\delta \mathbf{u}^T(\x_i-\x_k) > -\mathrm{margin}_{\bar \w}$,
%\end{equation}
which combined with (\ref{fdiff}) yields
\begin{equation*}
\hat f(\x_i)-\hat f(\x_k)\geq\mathrm{margin}_{\bar \w} + \delta \mathbf{u}^T(\x_i-\x_k) > \mathrm{margin}_{\bar \w}-\mathrm{margin}_{\bar \w} = 0.
\end{equation*}
%Thus, all pairwise orderings are preserved, that is, $\bar f(x_1)>\bar f(x_2)\longrightarrow \hat f(x_1)>\hat f(x_2)$.

\noindent \textbf{Proof of (ii)} We show that $\hat f$ has no ties $\hat f(\x_i)\neq \hat f(\x_k)$ for all $i, k$. This must be true with probability 1 over the choice of random vector $\mathbf{u}$.

Since we know that all pairwise inequalities are preserved, we need to ensure only that ties become untied through the perturbation $\mathbf{u}$. Thus, let us consider tied observations $\x_i$ and $\x_k$, so $\bar f(\x_i)=\bar f(\x_k)$. We need to show that they become untied: we need to show $|\hat f(\x_i)-\hat f(\x_k)|>0$. Consider $|\hat f(\x_i)-\hat f(\x_k)|$:
\begin{align*}
|\hat f(\x_i)-\hat f(\x_k)| &= \left|(\bar \w+\delta \mathbf{u})^T(\x_i-\x_k)\right|= \left|\bar{\mathbf{w}}^T(\x_i-\x_k) + \delta \mathbf{u}^T(\x_i-\x_k)\right| \\
&= |\delta|\left|\mathbf{u}^T(\x_i-\x_k)\right|.
\end{align*}
We now use the key assumption that no two observations are duplicates -- this implies that at least one entry of vector $\x_i-\x_k$ is nonzero. Further, since $\mathbf{u}$ is a random vector, the probability that it is orthogonal to vector $\x_i-\x_k$ is zero. So, 
%The vector $\x_{\text{diff}}=\x_i-\x_k$ is a fixed vector that is not identically zero. The vector $u$ is a random unit vector in $\mathcal{R}^d$. The probability that $u$ is orthogonal to $x_{\text{diff}}$, and therefore the probability that $u^T(x_1-x_2)=u^Tx_{\text{diff}}=0$, is zero. Thus 
with probability one with respect to the choice of $\mathbf{u}$, we have $\left|\mathbf{u}^T(\x_i-\x_k)\right|>0$. From the expression above,
\begin{equation*}
|\hat f(\x_i)-\hat f(\x_k)| = |\delta|\left|\mathbf{u}^T(\x_1-\x_2)\right| >0. %= |\delta|\left|\mathbf{u}^Tx_{\text{diff}}\right| > 0.
\end{equation*}

\noindent \textbf{Proof of (iii)} By our definitions, $\hat \w = \bar \w + \delta \mathbf{u}$, $\delta\leq  \min_{j\in\bar{J}} |w_j|$, and $\mathbf{u}$ is only nonzero in the components where $\bar\w$ is not 0. Each component of $\mathbf{u}$ is nonzero with probability 1. For component $j$ where $\bar{w}_j\neq 0$, we have $|\delta u_j|\leq \delta \|\mathbf{u}\|_2\leq \delta \leq \min_{j^{'}\in\bar{J}}\bar{w}_{j^{'}}\leq \bar{w}_j$ which means $|\hat{w}_j|=|\bar{w}_j+\delta u_j|>0$. So, for all components where $\bar{\w}$ is nonzero, we also have $\hat{\w}$ nonzero in those components. Further, for all components where $\bar{\w}$ is zero, we also have $\hat{\w}$ zero in those components. Thus $\|\bar{\w}\|_0=\|\hat{\w}\|_0$.

\end{proof}

The result below establishes the main result of the section, which is that if we optimize $\Gsub$, we get an optimizer of $\GRR$ even though it is a much more complex optimization problem to optimize $\GRR$ directly.

\begin{theorem}
\label{step2}
Given $f^*\in\mathrm{argmax}_f\Gsub(f)$, then $f^*\in\mathrm{argmax}_f\GRR(f)$.\\
This holds when there are no duplicated observations, $x_i\neq x_k$ $\forall i,k$ where $i\neq k$.
\end{theorem}
\begin{proof}
We will show that the assumption of Theorem~\ref{step1}, which says that $\GRR$ has a maximizer with no ties, is always true. This will give us the desired result. Let $\bar f\in\mathrm{argmax}_f\GRR(f)$. %, which exists by Lemma~\ref{optexist}.
Either $\bar f$ has no ties already, in which case there is nothing to prove, or it does have ties. If so, we can take its vector $\bar \w$ and perturb it using Lemma \ref{LemmaHighDim}. The resulting vector $\hat{\w}$ has no ties. We need only to show that $\hat{\w}$ also maximizes $\GRR$. To do this we will show $\GRR(\hat{f})\geq\GRR(\bar{f})$.

We know that % for $f(\x)=\w^T\x$, we have
\begin{align*}
\GRR(\bar f) &= \sum_{i=1}^n y_i\sum_{\ell=1}^{n} \mathbf{1}_{[\mathrm{ResolvedRank}({\bar f}(\x_i))=\ell-1]}\cdot a_\ell -c\|\bar{\w}\|_0\\
&= \sum_{i\in S_+} a_{(\mathrm{ResolvedRank}({\bar f}(\x_i))+1)}-c\|\bar{\w}\|_0,\\
\GRR(\hat f) &= \sum_{i=1}^n y_i\sum_{\ell=1}^{n} \mathbf{1}_{[\mathrm{ResolvedRank}({\hat f}(\x_i))=\ell-1]}\cdot a_\ell -c\|\hat{\w}\|_0\\
&= \sum_{i\in S_+} a_{(\mathrm{ResolvedRank}({\hat f}(\x_i))+1)} -c\|\hat{\w}\|_0,
\end{align*}
and $\|\bar{\w}\|_0 = \|\hat{\w}\|_0$ by Lemma~\ref{LemmaHighDim}. We know that $a_1\leq a_2\leq\dots\leq a_n$. Thus, as long as the ResolvedRanks of the positive observations according to $\hat{f}$ are the same or higher than their ResolvedRanks according to $\bar{f}$, we are done.

Consider the untied observations of $\bar{f}$, which are $\{i:\bar{f}(\x_i)\neq \bar{f}(\x_k)$ for any $k$$\}$. Those observations have ResolvedRank$(\bar{f}(\x_i))$ = ResolvedRank$(\hat{f}(\x_i))$ by Lemma \ref{LemmaHighDim}(i) which says that all pairwise orderings are preserved.

What remains is to consider the tied observations of $\bar{f}$, which are $\{i:\bar{f}(\x_i)= \bar{f}(\x_k)$ for some $k$$\}$. Consider a set of tied observations $\x_{\alpha},...,\x_{\zeta}$ where $f(\x_{\alpha})=...=f(\x_{\zeta})$. If their labels are all equal, $y_{\alpha}=...=y_{\zeta}$, then regardless of how they are permuted to create the ResolvedRank in either $\bar{f}$ or $\hat{f}$, the total contribution of those observations to the $\GRR$ will be the same. If the labels in the set differ, then $\bar{f}$ assigns ResolvedRanks pessimistically, so that the negatives all have ResolvedRanks above the positive (according to the definition of ResolvedRanks). This means that by perturbing the solution, $\hat{f}$ could potentially increase the ranks of some of these tied positive observations. In that case, some of the $a_{\ell}$'s of $\hat{f}$ become larger than those of $\bar{f}$. Thus, $\GRR(\hat{f})\geq \GRR(\bar{f})$ and we are done.

\end{proof}

The result in Theorem \ref{step2} shows why optimizing $\Gsub$ is sufficient to obtain the maximizer of $\GRR$. This provides the underpinning for use of the Subrank formulation.

\section{Empirical Discussion of Learning-To-Rank}\label{sectionexperiments}
Through our experiments with the Subrank formulation, we made several observations, which we will present empirical results to support below. 
\subsubsection*{Observation 1: There are some datasets where reranking can substantially improve the quality of the solution.}

We present comparative results on the performance of several baseline ranking methods methods, namely Logistic Regression (LR), Support Vector Machines (SVM), RankBoost (RB), and the P-Norm Push for $p=2$ and for the Subrank MIP formulations at 4 different levels of the cutoff $K$ for reranking. For the SVM, we tried regularization parameters $10^{-1}$, $10^{-2}$, $\ldots$, $10^{-6}$ and reported the best result. We chose datasets with the right level of imbalance so that not all of the top observations belong to a single class; this ensures that the rank statistics are meaningful at the top of the list. 
%Further details about the experimental setup are provided in the appendix. 
We used several datasets that are suitable for the type of method we are proposing, namely:
\begin{itemize}
\item ROC Flexibility: This dataset is designed specifically to show differences in rank statistics \citep{RudinFlex}.  Note that this dataset has ties, but the ties do not seem to influence the quality of the solution. (It is generally possible in practice to use the Subrank formulation even in the case of ties.) ($n=500$, $d=5$)
\item \textcolor{black}{Abalone19: This dataset is an imbalanced version of the Abalone dataset where the positive observations belong to class 19. It is available from the KEEL repository \citep{KEEL}. It contains information about sex, length, height, and weight, and the goal is to determine the age of the abalone (19). ($n=4174$, $d=8$)}
\item UIS from the UMass Aids Research Unit \citep{HosmerLe00}: This dataset contains information about each patient's age, race, depression level at admission, drug usage, number of prior drug treatments, and current treatment, and the label represents whether the patient remained drug free for 12 months afterwards. ($n=575$, $d=8$)
\item Travel: This dataset is from a survey of transportation uses between three Australian cities \citep{HosmerLe00}. It contains information about what modes of traffic are used (e.g., public bus, airplane, train, car) which is what we aim to predict, and features include the travel time, waiting time at the terminal, the cost of transportation, the commuters' household income level, and the size of the party involved in the commute.  ($n=840$, $d=7$)
\item NHANES (physical activity):  This dataset contains health information about patients including physical activity levels, height, weight, age, gender, blood pressure, marital status, cholesterol, etc. \citep{HosmerLe00}. The goal is to predict whether the person is considered to be obese.  ($n=600$, $d=21$)
\item Pima Indians Diabetes, from the National Institute of Diabetes and Digestive and Kidney Diseases, available from the UCI Machine Learning Repository \citep{Bache+Lichman:2013}: The goal is to predict whether a woman will test positive for diabetes during her pregnancy, based on measurements of her blood glucose concentration in an oral glucose tolerance test, her blood pressure, body mass index, age, and other characteristics.  ($n=768$, $d=8$)
\item \textcolor{black}{Gaussians: This is a synthetic 2 dimensional dataset, with 1250 points subsampled from a population containing two big clumps of  training examples, each entry of each observation drawn from a normal distribution with variance 0.5, where the positive clump (3000 points) was generated with mean (0,1), and the negative clump (3000 points) was generated with mean
(0,0). These bigger clumps are designed to dominate the WRS. In addition, there is a smaller 10 point negative clump generated with mean (10,1) and noise components each drawn from a normal with standard deviation 0.05, and a positive clump of 200 points generated with
mean (0,-3) and noise drawn with standard deviation 0.05. Note that we do not expect the ``flipping" to occur here as it did in Section \ref{subsec:why} since we are using DCG, which is much more difficult to distinguish from WRS than a steeper rank statistic. ($n=1250$, $d=2$) }
\end{itemize}
For the MIP-based methods, we used logistic regression as the base ranker, and the reranker was learned from the top $K$. We varied $K$ between 50, 100, 150, and we also used the full list. 
An exception is made for the Abalone19 data set, for which $K$ varies between 250, 500 and 750 instead because Abalone19 is a highly imbalanced data set.
We stopped the computation after 2 hours for each trial (1 hour for the ROC flexibility dataset), which gives a higher chance for the lower-$K$ rerankers to solve to optimality. 
%Note that this means that the computation over the full list will thus be an approximate solution. 
Most of the K=50 experiments for the ROC flexibility dataset solved to optimality within 5 minutes.
The reported means and standard deviations were computed over 10 randomly chosen training and test splits, where the same splits were used for all datasets.
%The set of 10 runs for each dataset took at most 20 hours total. 
%which is at most 100 hours of total computation time per dataset for the full table. 
We chose to evaluate according to the DCG measure as it is used heavily in information retrieval applications \citep{Jarvelin00}. \citet{AlMaskari07} report that DCG is similar to the way humans evaluate search results, based on a study of user satisfaction. We used $C=10^{-3}$ for the ROC Flexibility dataset, and $C=10^{-4}$ for the other datasets. Note that for the DCG measure in particular, it is difficult to see a large improvement; for instance even on the extreme experiment in Section \ref{subsec:why} the improvement in DCG from flipping the classifier completely upside down was only 16\%. %(the calculation is (316-284)/284=11.27\%).

Table \ref{Table:gooddatasets} shows the results of our experiments, where we highlighted the best algorithm for each dataset on both training and test in bold, and used italics to represent test set results that are not statistically significantly worse than the best algorithm according to a matched pairs t-test with significance level $\alpha=0.05$. In terms of predictive performance, the smaller $K$ models performed consistently well on these data, achieving the best test performance on all of these datasets. On some of the datasets, we see a $\sim$10\% average performance improvement from reranking. (The magnitude is not too much different as from the experiment in Section \ref{subsec:why} where the classifier flips upside down.)
On the Travel dataset in particular, the $K$=50 reranking model had superior results over all of the baselines uniformly across all 10 trials. 

%In more detail, on the ROC Flexibility dataset, $K$=100 and $K$=50 performed similarly, beating all other baselines. For the UIS dataset, $K$=100 was often best in both training and test, though sometimes K=50 was better. 

%Note that the performance results reported here could vary by machine, and could change as the solvers continue to improve (their improvement has been dramatic over the last decade), and could vary by the length of time used for the reranking step.
%For the ROC flexibility dataset, the $K=50$ method performs 2-10\% better on the datasets

\begin{table}[h!]
%\centering
\begin{sideways}
\begin{minipage}{22cm}
%\scalebox{1}{
\caption{Datasets for which reranking can make a difference\label{Table:gooddatasets}}

\begin{tabular}{ccccccccccc}
\hline\hline
& & \multicolumn{4}{c}{Baseline methods}& & \multicolumn{4}{c}{MIP-based methods}\\
\cline{3-6}
\cline{8-11}
\multicolumn{2}{c}{Dataset}& LR& SVM& RB& P-norm Push%\footnote{We use 2-norm for P-norm Push on all datasets.}
& &$K=50$& $K=100$&  $K=150$& Full List\\
\hline
\multirow{2}{*}{ROC%\footnote{The regularization constant $c$ is set to $10^{-3}$ for this dataset.}
}& train&$31.21\pm 1.65$ &$30.94\pm 1.57$ &$29.00\pm 1.39$ &$31.33\pm 31.43$&  &$\mathbf{31.96\pm 1.32}$ &$31.84\pm 1.36$ &$31.65\pm 1.12$  &$28.43\pm 1.80$ \\
&test&$31.35\pm 1.48$ &$31.10\pm 1.63$ &$29.57\pm 1.43$ &$31.43\pm 1.61$& &$\mathbf{32.16\pm 1.31}$ &$\mathit{32.09\pm 1.31}$ &$\mathit{31.74\pm 1.65}$ &$28.96\pm 2.40$ \\ 
\hline
\multirow{2}{*}{Abalone19
\footnote{We use $K=250$, $K=500$ and $K=750$ for this data set because it is highly imbalanced.}
}& train& $3.63\pm 0.43$& $3.41\pm 0.47$& $3.40\pm 0.65$& $3.44\pm 0.47$& & $\mathbf{4.89\pm 0.58}$& $4.45\pm 0.50$& $4.13\pm 0.53$& $2.54\pm 0.35$\\
& test& $2.96\pm 0.42$& $3.02\pm 0.53$& $2.66\pm 0.49$& $3.03\pm 0.50$& & $\mathbf{3.08\pm 0.49}$& $\mathit{2.89\pm 0.35}$& $\mathit{2.76\pm 0.38}$& $2.42\pm 0.48$\\
\hline
\multirow{2}{*}{UIS}& train& $18.86\pm 1.32$& $18.46\pm 1.38$& $19.44\pm 1.44$& $18.78\pm 1.40$& &$\mathbf{20.45\pm 1.23}$& $19.76\pm 1.27$& $19.26\pm 1.05$& $18.84\pm 1.44$\\
& test& $17.88\pm 1.11$& $17.81\pm 1.21$& $17.70\pm 1.40$& $\mathit{17.89\pm 1.13}$& & $\mathit{18.00\pm 1.31}$& $\mathbf{18.64\pm 1.51}$& $\mathit{17.79\pm 1.73}$& $\mathit{17.89\pm 0.67}$\\
\hline
\multirow{2}{*}{Travel}& train& $28.16\pm 1.60$& $27.59\pm 1.61$& $26.57\pm 1.60$& $28.09\pm 1.62$& & $\mathbf{28.30\pm 1.63}$& $28.24\pm 1.56$& $27.12\pm 1.45$& $26.94\pm 1.36$\\
& test& $27.32\pm 1.70$& $26.81\pm 1.76$& $24.95\pm 1.63$& $27.24\pm 1.66$& & $\mathbf{27.61\pm 1.70}$& $\mathit{27.39\pm 1.60}$& $26.00\pm 2.26$& $26.31\pm 1.83$\\ 
\hline
\multirow{2}{*}{NHANES
%\footnote{Only the first 600 observations were used.}
}& train& $14.69\pm 1.63$& $13.83\pm 1.87$& $13.75\pm 2.01$& $14.46\pm 1.57$& & $\mathbf{15.48\pm 1.69}$& $15.02\pm 1.93$& $14.73\pm 1.59$& $13.87\pm 1.25$\\
& test& $\mathit{13.06\pm 1.74}$& $\mathit{12.98\pm 1.79}$& $12.10\pm 1.75$& $\mathit{13.18\pm 1.82}$& &$\mathbf{13.26\pm 1.50}$& $\mathit{12.71\pm 1.61}$& $\mathit{12.94\pm 1.55}$& $\mathit{13.09\pm 1.82}$\\
\hline
\multirow{2}{*}{Pima}& train& $35.50\pm 1.66$& $35.30\pm 1.61$& $\mathbf{35.80\pm 1.44}$& $35.64\pm 1.67$& & $35.75\pm 1.67$& $35.43\pm 1.69$& $34.85\pm 1.96$& $34.77\pm 2.03$\\
& test& $\mathit{34.18\pm 1.81}$& $\mathit{34.03\pm 1.83}$& $33.83\pm 1.65$& $\mathit{34.24\pm 1.82}$& & $\mathbf{34.44\pm 1.76}$& $33.56\pm 1.87$& $33.72\pm 2.21$& $\mathit{33.65\pm 2.01}$\\
\hline
\multirow{2}{*}{Gaussians}& train& $69.25\pm 2.70$&$69.28\pm 2.70$ & $71.31\pm 2.15$& $69.24\pm 2.70$& & $71.71\pm 2.22$& $\mathbf{71.76\pm 2.18}$& $71.58\pm 2.29$& $64.70\pm 2.53$\\
& test& $64.69\pm 2.45$&$64.73\pm 2.45$ & $67.13\pm 2.06$& $64.65\pm 2.43$& & $\mathbf{68.03\pm 2.27}$& $\mathit{67.91\pm 2.27}$& $\mathit{67.79\pm 2.30}$& $59.89\pm 2.26$\\
\hline\hline
\end{tabular}
%}
\end{minipage}
\end{sideways}
%%%%
\begin{sideways}
\begin{minipage}{17cm}
\vspace*{20pt}
\caption{Datasets for which reranking does not make a difference\label{Table:baddatasets}}
%\scalebox{1}{
\begin{tabular}{ccccccccccc}
\hline\hline
& & \multicolumn{4}{c}{Baseline methods}& & \multicolumn{4}{c}{MIP-based methods}\\
\cline{3-6}
\cline{8-11}
\multicolumn{2}{c}{Dataset}& LR& SVM& RB& P-norm Push& &$K=50$& $K=100$& $K=150$& Full MIO\\
\hline
\multirow{2}{*}{Haberman}& train& $12.94\pm 1.07$& $12.95\pm 1.06$& $\mathbf{13.95\pm 1.24}$& $12.94\pm 1.06$& & $13.10\pm 1.19$& $13.02\pm 1.20$& ---\footnote{There are only 153 observations in the training data for the Haberman survival dataset. We did not do $K=150$ because at that point it makes sense to run the MIP on the full dataset.} & $13.13\pm 0.92$\\ 
& test& $\mathbf{12.82\pm 1.15}$& $\mathit{12.63\pm 1.15}$& $12.01\pm 1.06$& $\mathit{12.64\pm 1.09}$& & $\mathit{12.64\pm 1.47}$& $\mathit{12.80\pm 1.13}$& ---& $12.46\pm 1.05$\\
\hline
\multirow{2}{*}{Polypharm}& train& $\mathbf{19.43\pm 1.42}$& $19.24\pm 1.53$& $19.11\pm 1.55$& $18.73\pm 1.25$& & $19.21\pm 0.93$& $18.63\pm 1.25$&$18.04\pm 1.04$ & $18.16\pm 1.36$\\
& test& $\mathit{17.23\pm 1.63}$& $\mathit{17.71\pm 1.56}$& $\mathit{17.40\pm 1.70}$& $\mathbf{18.05\pm 1.33}$& & $17.37\pm 1.33$& $\mathit{17.59\pm 1.25}$&$\mathit{17.16\pm 1.99}$ & $17.22\pm 1.57$\\
\hline
\multirow{2}{*}{Glow500}& train& $17.26\pm 1.16$& $16.52\pm 1.41$& $17.23\pm 1.25$& $17.02\pm 1.44$& & $\mathbf{18.20\pm 1.24}$& $17.70\pm 1.22$& $17.08\pm 1.40$& $16.69\pm 1.28$\\
& test& $16.68\pm 1.18$& $\mathbf{17.26\pm 1.15}$& $\mathit{17.01\pm 0.69}$& $\mathit{17.24\pm 1.27}$& & $\mathit{16.79\pm 1.09}$& $\mathit{17.10\pm 1.66}$& $\mathit{16.78\pm 1.33}$& $16.48\pm 1.35$\\
\hline\hline
\end{tabular}
%}
\end{minipage}
\end{sideways}
%\end{sideways}
\end{table}

The work of \citet{ChangEtAl2012} shows the benefits of carrying the computation to optimality on a specialized application of MIP learning-to-rank for reverse-engineering product quality rankings.

%%%%%%%%
\subsubsection*{Observation 2: There is a tradeoff between computation and quality of solution.} 
If the number of elements to rerank (denoted by $K$) is too small, the solution will not generalize as well. Theoretical results of \cite{Rudin09} suggest that there is a tradeoff between how well we can generalize and how much the rank statistic is focused on the top of the ranked list. The main result of that work shows that if the rank statistic concentrates very much at the top of the list (like, for instance, the mean reciprocal rank statistic) then we require more observations in order to generalize well. If the number of observations is too small, learning-to-rank methods may not be beneficial over traditional learning methods like logistic regression. Further, if the number of observations is too small, then the variation from training to test will be much larger than the gain in training error from using the correct rank statistic; again in that case, learning-to-rank would not be beneficial.

If the number of elements $K$ is too large, we will not be able to sufficiently solve the reranking problem within the allotted time, and the solution again could suffer. This reinforces our point that we should not refrain from solving hard problems, particularly on the scale of reranking, but certain hard problems are harder than others and the computation needs to be done carefully. 

Again consider Table \ref{Table:gooddatasets}. Note that the $K=50$ and $K=100$ rerankers perform consistently well on these datasets, both in training and in testing. However, if $K$ is set too large, the optimization on the training set will not be able to be solved close to optimality in the allotted time, and the quality of the solution will be degraded. This is an explicit tradeoff between computation and the quality of the solution.

%%%%%%
\subsubsection*{Observation 3: There are some datasets for which the variance of the result is larger than the differences in the rank statistics themselves.}
These are cases where better relative training values do not necessarily lead to better relative test values. In these cases we do not think it is worthwhile to use ranking algorithms at all, let alone \textit{reranking} algorithms. For these datasets, logistic regression may suffice. The cases where reranking/ranking makes a difference are cases where the variance of the training and test values are small enough that we can reliably distinguish between the different rank statistics.

We present results on three datasets in Table \ref{Table:baddatasets}, computed in the same way as the results in Table \ref{Table:gooddatasets}, for which various things have gone wrong, such as the optimizer not being able to achieve the best result on the training set, but even worse, the results are inconsistent between training and test. The algorithm that optimizes best over the training set is not the same algorithm that achieves the best out-of-sample quality. These are cases where the algorithms do not generalize well enough so that a ranking algorithm is needed. The datasets used here are the Haberman survival dataset from the UCI Machine Learning Repository \citep{Bache+Lichman:2013} ($n=300$, $d=3$), Poly-pharmacy study on drug consumption \citep{HosmerLe00} ($n=500$, $d=13$), and data from the GLOW study on fracture risk \citep{HosmerLe00} ($n=500$, $d=14$).

\subsubsection*{Observation 4: As long as the margin parameter $\varepsilon$ is sufficiently small without being too small so that the solver will not recognize it, the quality of the solution is maintained. The regularization parameter $C$ also can have an influence on the quality of the solution, and it is useful not to over-regularize or under-regularize.}

Note that if $\varepsilon$ is too large, the solver will not be able to force all of the inequalities to be strictly satisfied with margin $\varepsilon$. This could force many good solutions to be considered infeasible and this may ruin the quality of the solution. It could also cause problems with convergence of the optimization problem. When $\varepsilon$ is smaller, it increases the size of the feasible solution space, so the problem is easier to solve. On the other hand, if $\varepsilon$ is too small, the solver will have trouble recognizing the inequality and may have numerical problems. 

In Table \ref{TableEpsi} we show what happens when the value of $\varepsilon$ is varied on two of our datasets. 
%If $\varepsilon$ is too large, it makes the constraints more difficult (or impossible) to satisfy, which could cause problems both with convergence of the optimization problem, and with the quality of the solution. 
We can see from Table \ref{TableEpsi} that as $\varepsilon$ decreases by orders of magnitude the solution generally improves, but then at some point degrades. For the ROC Flexibility data, the $\varepsilon=10^{-5}$ setting consistently performed better than the $\varepsilon=10^{-6}$ setting over all 10 trials in both training and test. A similar observation holds for UIS, in that the $\varepsilon=10^{-5}$ setting was able to optimize better than the $\varepsilon=10^{-6}$ setting over all 10 trials on the training set.

\begin{table}[h!]
\centering
\caption{Different selections of $\varepsilon$\label{TableEpsi}}
\begin{tabular}{cccccccc}
\hline\hline
\multicolumn{2}{c}{Dataset}& $10^{-1}$& $10^{-2}$& $10^{-3}$& $10^{-4}$& $10^{-5}$& $10^{-6}$\\
\hline
\multirow{2}{*}{ROC\footnotemark}& train&$31.62\pm 1.25$ &$31.85\pm 1.26$ & $31.93\pm 1.36$& $31.84\pm 1.36$& $32.02\pm 1.30$& $31.58\pm 1.26$\\
& test&$31.59\pm 2.07$ &$31.91\pm 1.57$ & $32.10\pm 1.28$& $32.09\pm 1.31$& $32.21\pm 1.32$& $31.74\pm 1.26$\\
\hline
\multirow{2}{*}{UIS}& train&$19.70\pm 1.23$ & $19.73\pm 1.24$& $19.80\pm 1.09$& $19.76\pm 1.27$& $19.73\pm 1.17$& $19.08\pm 1.40$\\
& test&$18.40\pm 1.09$ & $18.03\pm 1.06$& $18.34\pm 1.20$& $18.64\pm 1.51$& $17.88\pm 1.66$& $18.23\pm 0.75$\\
\hline\hline
\end{tabular}
\end{table}
\footnotetext{{The regularization constant $C$ is set to $10^{-3}$ for this dataset.}}

Table \ref{TableRegulariz} shows the training and test performance as the regularization parameter $C$ is varied over several orders of magnitude. As one would expect, a small amount of regularization helps performance, but too much regularization hurts performance as we start to sacrifice prediction quality for sparseness.

\begin{table}[h!]
\centering
\caption{Training and test performance for varying values of regularization parameter $C$. \label{TableRegulariz}}
\begin{tabular}{cccccccc}
\hline\hline
\multicolumn{2}{c}{Dataset}& $C=10^{-1}$& $C=10^{-2}$& $C=10^{-3}$& $C=10^{-4}$& $C=10^{-5}$& $C=10^{-6}$\\
\hline
\multirow{2}{*}{ROC}& train& $31.31\pm 1.52$& $31.31\pm 1.72$& $31.84\pm 1.36$& $31.94\pm 1.20$& $32.02\pm 1.30$&$31.86\pm 1.35$ \\
& test& $31.35\pm 1.48$& $31.30\pm 1.57$& $32.09\pm 1.31$& $32.06\pm 1.53$& $32.21\pm 1.32$&$32.01\pm 1.35$ \\
\hline
\multirow{2}{*}{UIS}& train& $19.15\pm 1.24$& $19.36\pm 1.02$& $19.69\pm 1.01$& $19.76\pm 1.27$& $19.89\pm 1.24$& $19.68\pm 1.08$\\
& test& $17.94\pm 1.11$& $18.15\pm 1.54$& $17.92\pm 1.57$& $18.64\pm 1.51$& $17.94\pm 1.12$& $17.74\pm 1.76$\\
\hline\hline
\end{tabular}
\end{table}

\subsubsection*{Observation 5: Proving optimality takes longer than finding a reasonable solution.} 
Figures \ref{ROCTime} shows the objective values and the upper bound on the optimality gap over time for four folds of the ROC Flexibility dataset, where $K$ is 100 and $C$ is $10^{-4}$. Figure \ref{UISTime} shows the analogous plots for the UIS dataset. Usually a good solution is found within a few minutes, whereas proving optimality of the solution takes much longer. We do not require a proof of optimality to use the solution.

\begin{figure}\caption{\label{ROCTime}Objective values and optimality gap over time for ROC Flexibility dataset}
\begin{tabular}{c|c}
Fold 1 & Fold 2\\
\includegraphics[width=2.5in,height=2in]{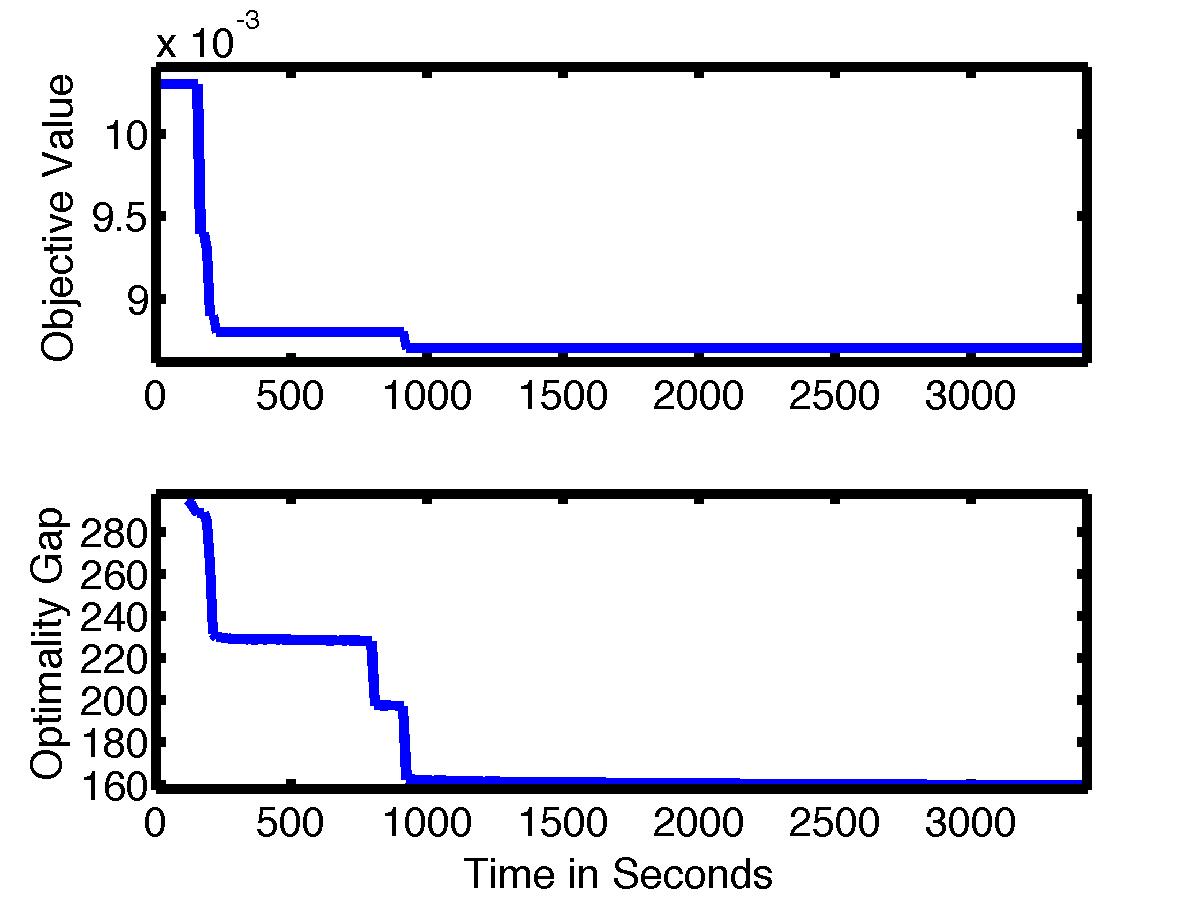}&
\includegraphics[width=2.5in,height=2in]{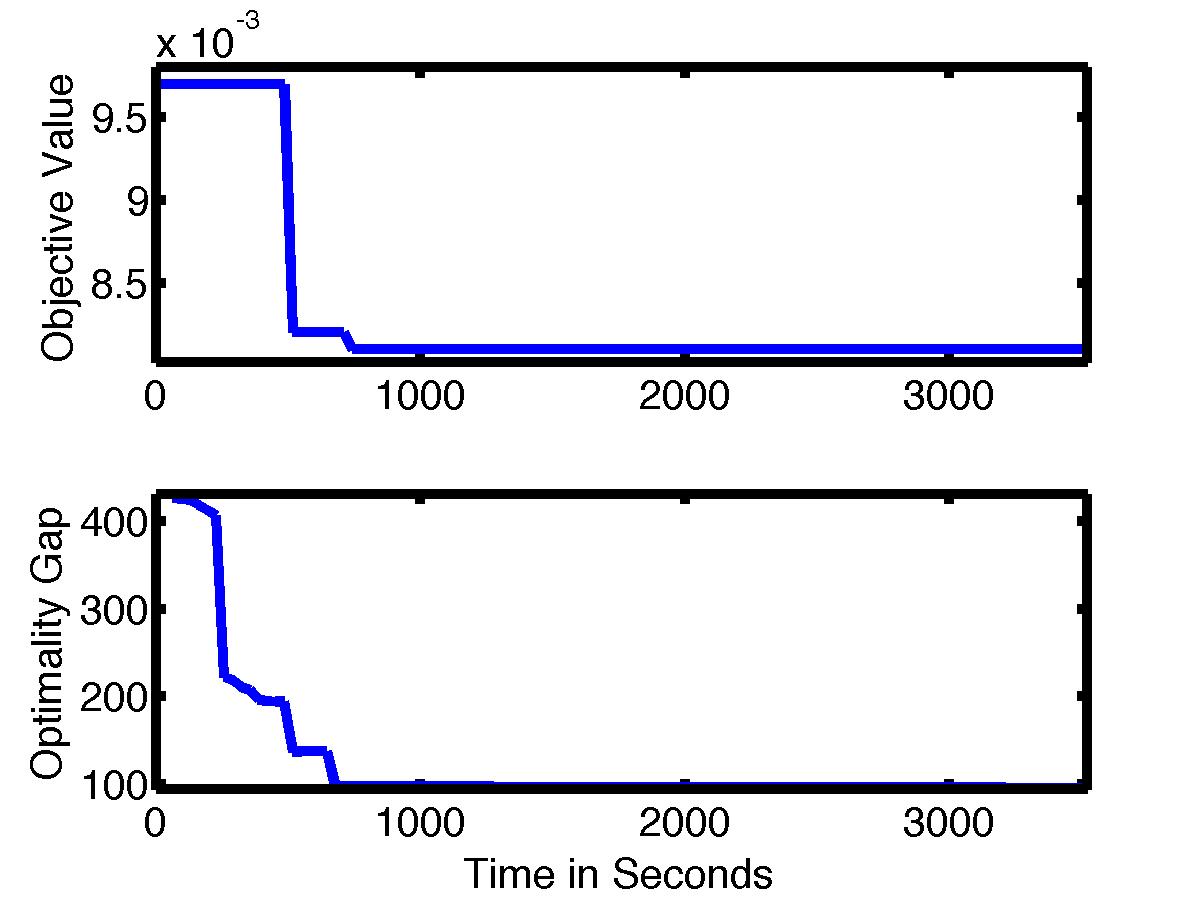}\\
\hline\\
Fold 3 & Fold 4\\
\includegraphics[width=2.5in,height=2in]{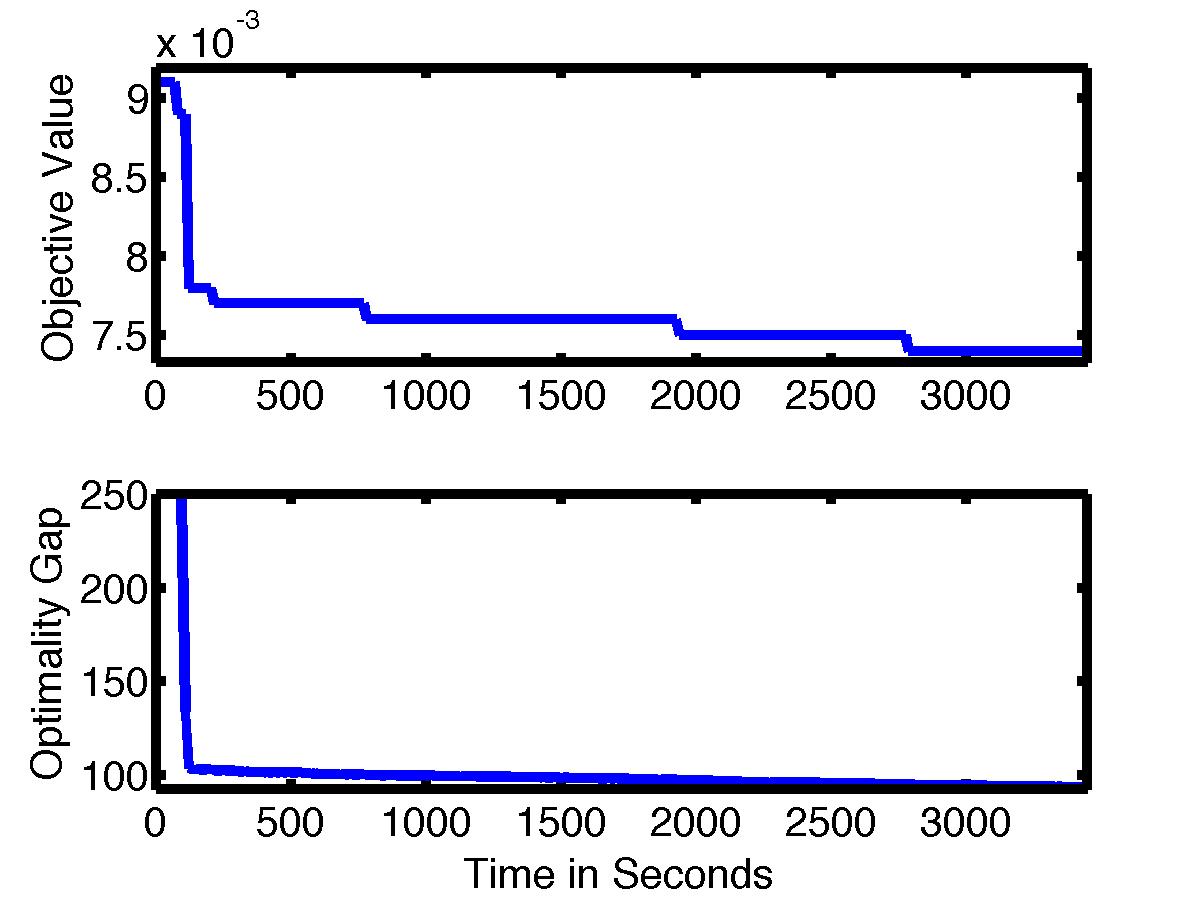}&
\includegraphics[width=2.5in,height=2in]{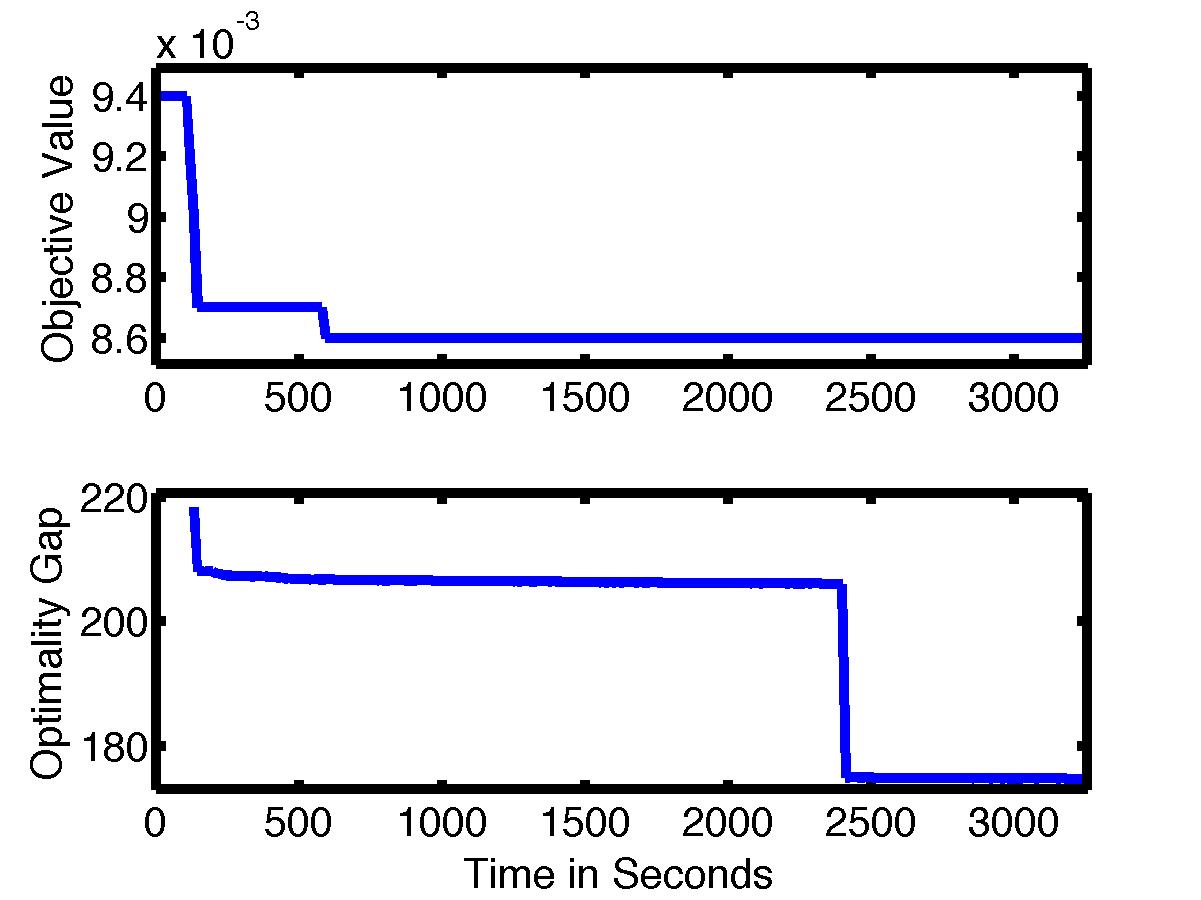}
\end{tabular}
\end{figure}

\begin{figure}\caption{\label{UISTime}Objective values and optimality gap over time for UIS dataset}
\begin{tabular}{c|c}
Fold 1 & Fold 2\\
\includegraphics[width=2.5in,height=2in]{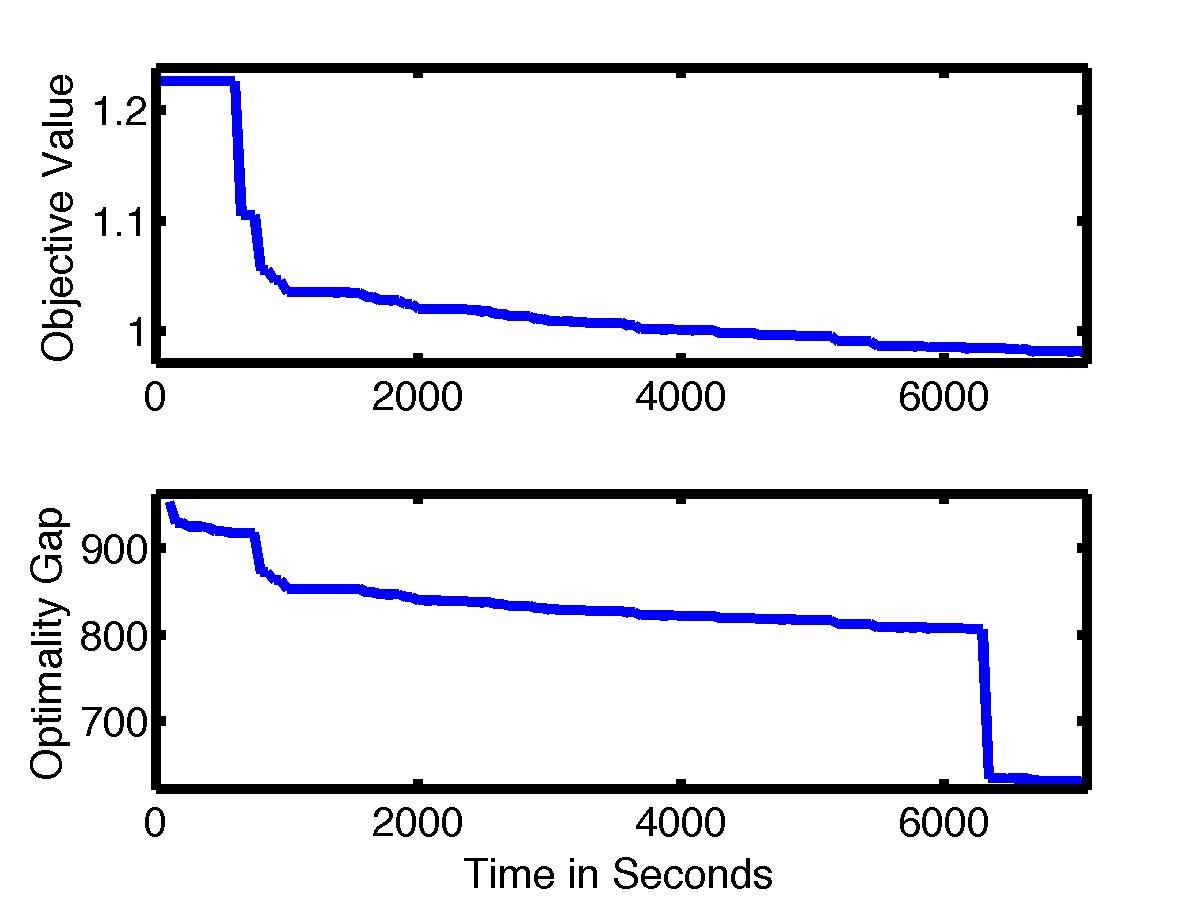}&
\includegraphics[width=2.5in,height=2in]{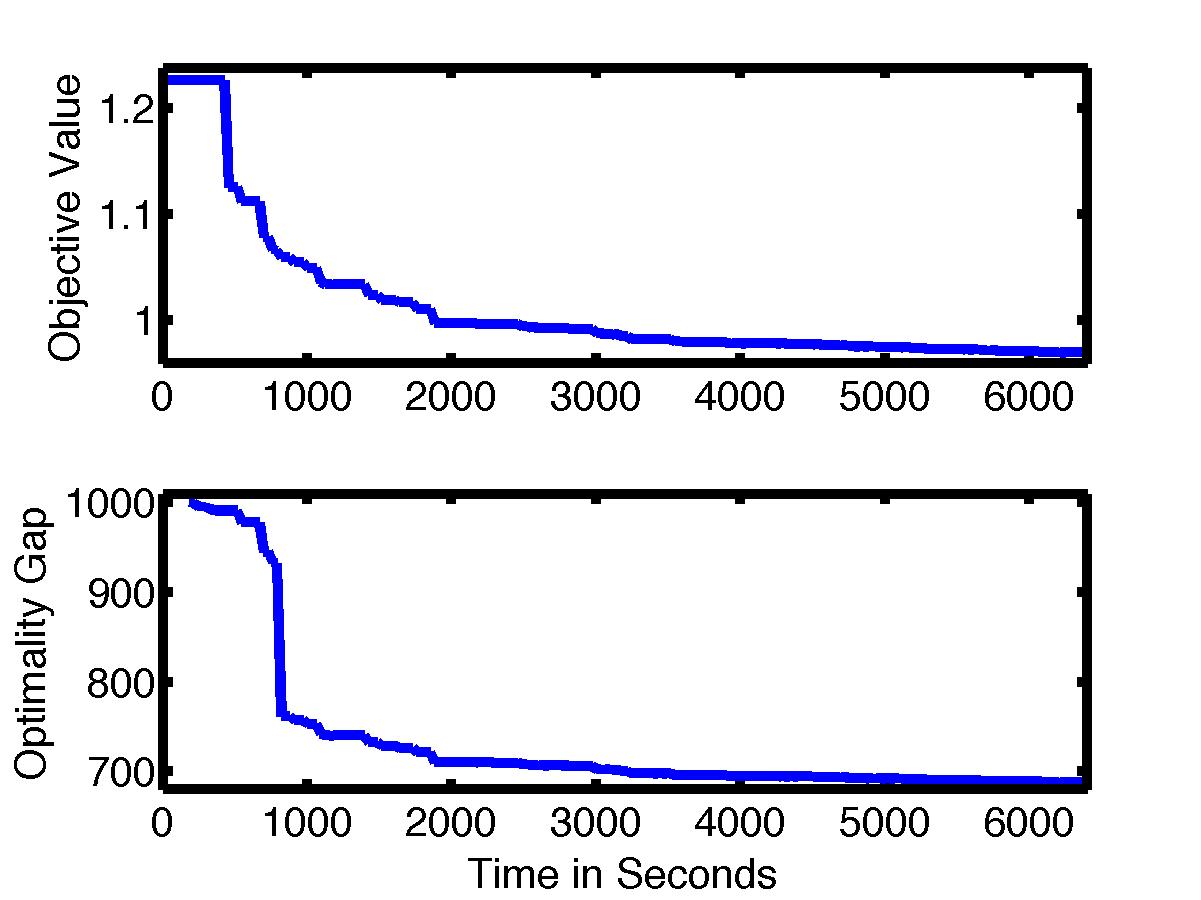}\\
\hline\\
Fold 3 & Fold 4\\
\includegraphics[width=2.5in,height=2in]{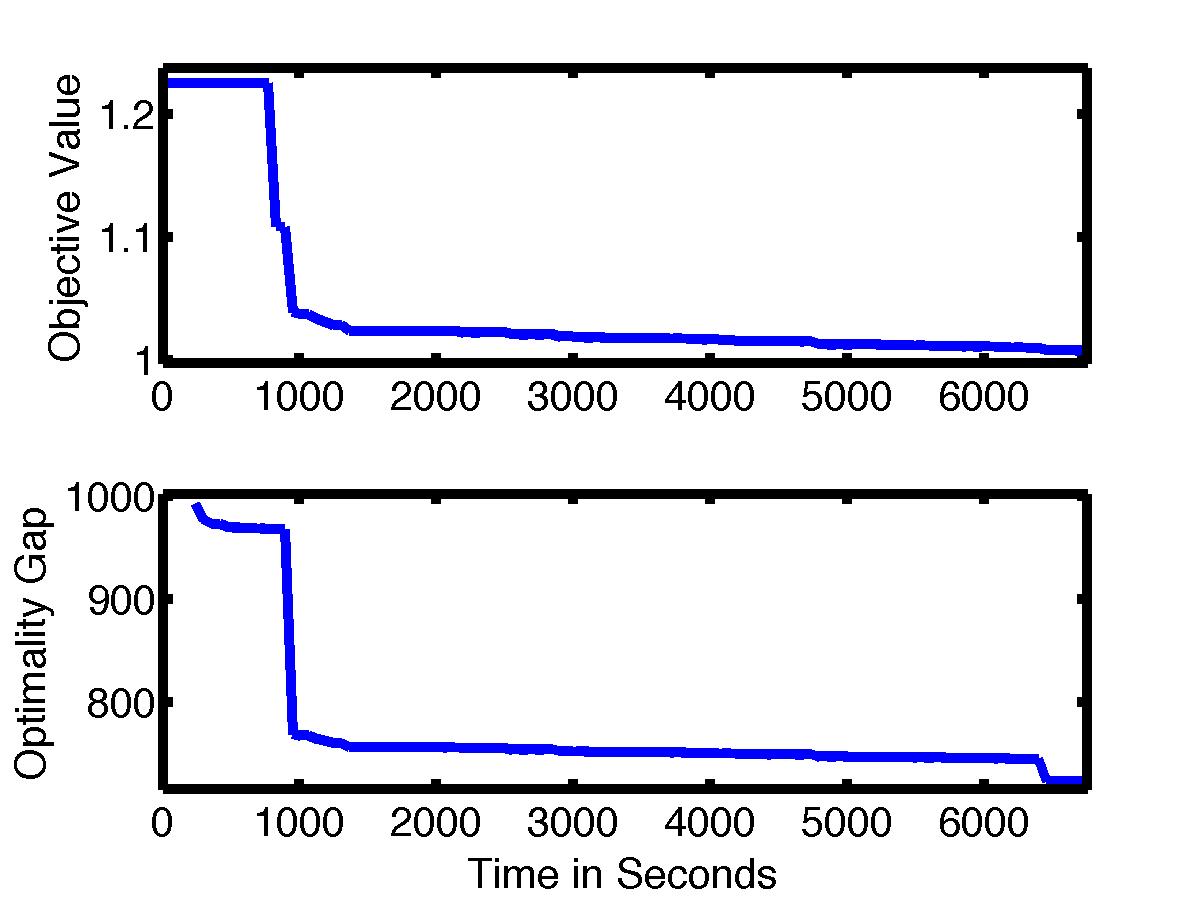}&
\includegraphics[width=2.5in,height=2in]{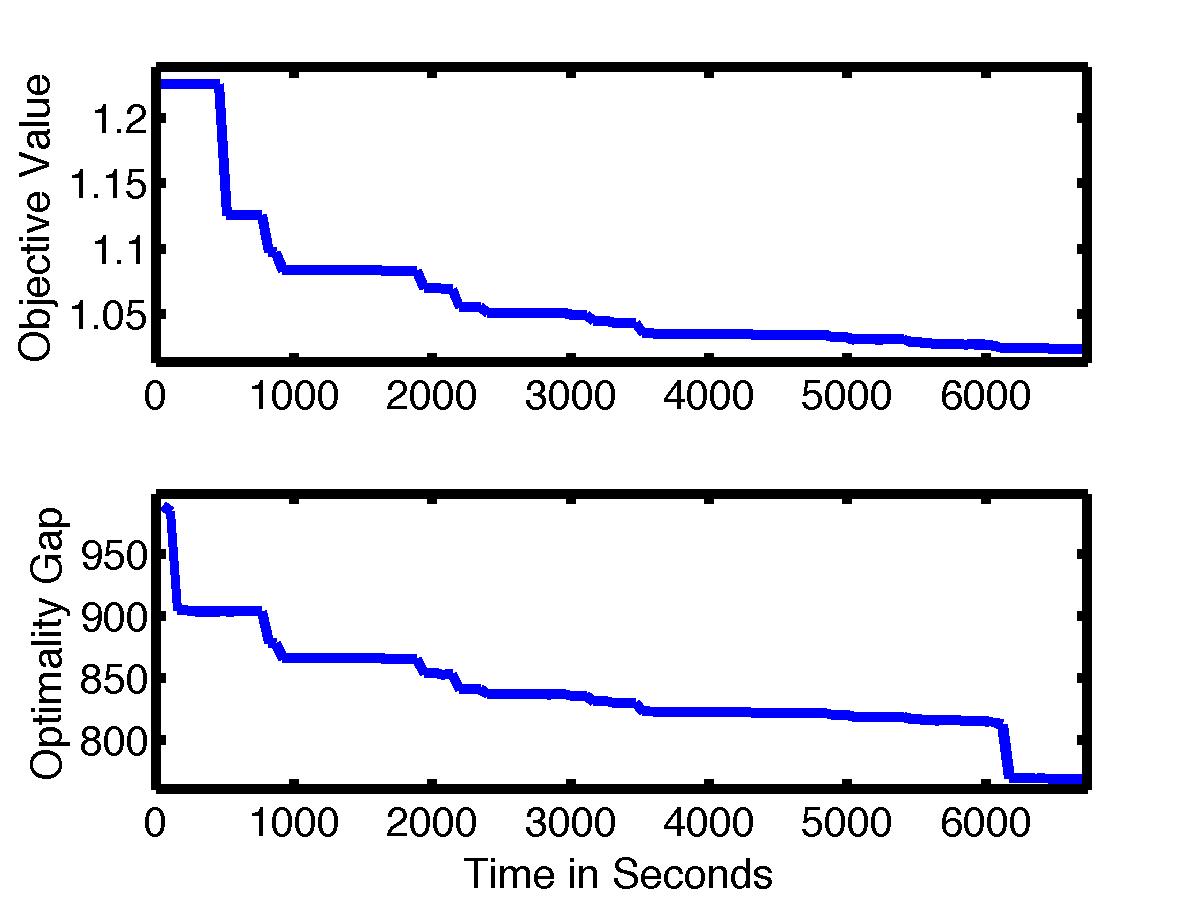}
\end{tabular}
\end{figure}

%------

\section{Conclusion}
\label{conclusion}
As shown through our discussion, using a computationally expensive reranking step may help to improve the quality of the solution for reranking problems. This can be useful in application domains such as maintenance prioritization and drug discovery where the extra time spent in obtaining the best possible solution can be very worthwhile. We proved an analytical reduction from the problem that we really want to solve (the ResolvedRank formulation) to a much more computationally tractable problem (the Subrank formulation). Through an experimental discussion, we explicitly showed the tradeoff between computation and the quality of the solution. 

\section*{Acknowledgements} We would like to thank Allison Chang, our co-author on a technical report that inspired this work.
This research was supported by the National Science Foundation under Grant No IIS-1053407 to C. Rudin, and an undergraduate exchange fellowship for Y. Wang through the Tsinghua-MIT-CUHK Research Center for Theoretical Computer Science.

\bibliographystyle{spbasic}
\bibliography{RankingMIP2}

\appendix
\section{Appendix}

\subsection{Formulation to Maximize Regularized AUC}\label{AUCopt}
Again we want to have $z_{ik}=1$ if $\w^T\x_i>\w^T\x_k$ and  $z_{ik}=0$ otherwise. We want to maximize the sum of the $z_{ik}$'s which is the number of correctly ranked positive-negative pairs. If $\w^T\x_i-\w^T\x_k\leq \varepsilon$ then it is not considered to be correctly ranked. So we need to impose that $z_{ik}$ is 0 when $\w^T\x_i-\w^T\x_k- \varepsilon\leq 0$; that is, when 1 plus this quantity is less than 1, $z_{ik}$ is 0. Thus, we impose \[z_{ik}\leq 1+\w^T\x_i-\w^T\x_k-\varepsilon.\]
Regularization is included as usual. The formulation is:
\begin{align}
\label{auc_formulation}\notag
%G_{\mathrm{AUC}}: 
\max_{\w,\gamma_j,z_{ik}\forall j,i,k} \quad &\sum_{i\in S_+}\sum_{k\in S_-} z_{ik}-C\sum_j\gamma_j\\\notag
\text{s.t.}\quad &z_{ik}\leq \w^T(\x_i-\x_k)+1-\varepsilon, \quad\forall i\in S_+,k\in S_-,\\\notag
%&v_i = w^Tx_i, \quad\forall i\in S_+,\\\notag
%&v_k = w^Tx_k, \quad\forall k\in S_-,\\\notag
&\gamma_j\geq w_j, \quad\forall j=1,\dots,d,\\\notag
&\gamma_j\geq -w_j, \quad\forall j=1,\dots,d,\\\notag
&-1\leq w_j\leq 1, \quad\forall j=1,\dots,d,\\\notag
&z_{ik},\gamma_j\in \{0,1\}, \quad\forall i\in S_+,k\in S_-, j\in\{1,...d\}.
\end{align}

%-------

\subsection{Ranking for the General Pairwise Preference Case}
\label{Pairwiseopt}
RankBoost \citep{Freund03} was designed to handle any pairwise preference information. Here we present an exact, regularized version of RankBoost's objective. Define the labels as $\pi(\x_i,\x_k)=\pi_{ik}$, where $\pi_{ik}$ is 1 if $\x_i$ should be ranked higher than $\x_k$. If $\pi_{ik}=0$ there is no information about the relative ranking of $i$ to $k$.
Then we try to maximize the number of pairs for which the model is able to rank $\x_i$ above $\x_k$ and for which the label for the pair is $\pi_{ik}=1$:
%\begin{displaymath}
%\pi_{ik} = \begin{cases}1, \quad\text{if $x_i$ is ranked strictly higher than $x_k$,}\\0, \quad\text{otherwise.}\end{cases}
%\end{displaymath}
%Also let
%\begin{displaymath}
%\Pi = \sum_{i=1}^n\sum_{k=1}^n\pi_{ik}.
%\end{displaymath}
\begin{displaymath}
%\mathrm{AUC}_{\pi}(f)=\frac{1}{\Pi}
\textrm{NumAgreedPairs}=\sum_{i=1}^n\sum_{k=1}^n \pi_{ik}\mathbf{1}_{[f(\x_i)>f(\x_k)]}.
\end{displaymath}
%This is related to Kendall's $\tau$ coefficient~\citep{Kendall38}. The highest possible value of $\mathrm{AUC}_{\pi}(f)$ is 1. We achieve this value if our scoring function $f$ satisfies $f(x_i)>f(x_k)$ for all pairs $(x_i,x_k)$ such that $\pi_{ik}=1$. We can use the following MIO formulation to maximize $\mathrm{AUC}_{\pi}$:
We will maximize a regularized version of this, as follows:
\begin{align*}
%G_{\textrm{pairwise}}(\varepsilon): 
\max_{\w,\gamma_j,z_{ik}\forall j,i,k} \quad &\sum_{i=1}^n\sum_{k=1}^n \pi_{ik}z_{ik}-C\sum_j\gamma_j\\
\text{s.t.}\quad &z_{ik}\leq \w^T(\x_i-\x_k)+1-\varepsilon, \quad\forall i,k=1,\dots,n,\\
%&z_{ii} = 0 \quad\forall i\\
%&z_{ik}+z_{ki}\leq 1 \quad\forall i<k\\
&-1\leq w_j\leq 1, \quad\forall j=1,\dots,d,\\
&\gamma_j\geq w_j, \quad\forall j=1,\dots,d,\\
&\gamma_j\geq -w_j, \quad\forall j=1,\dots,d,\\
&z_{ik},\gamma_j\in \{0,1\}, \quad\forall i,k=1,\dots,n,\;\; j\in\{1,...d\}.
\end{align*}

By special choices of $\pi$, the pairwise rank statistic can be made to include multipartite ranking~\citep{Rajaram05}, which can be similar to ordinal regression. In this case we have several classes, where observations in one class should be ranked above (or below) all the observations in another class.
% For instance if observations from on class  should  should be ranked higher than those in another class (class B), we can have the following for all pairs of observations
%In particular, suppose that there are $C$ classes and that we would like Class~1 to be ranked above Class~2, Class~2 above Class~3, and so on. Denoting the class of observation $x_i$ by $\text{Class}(x_i)$, we would set
\begin{displaymath}
\pi_{ik}=\begin{cases}1 \quad\text{if observations in Class$(\x_i)$ should be ranked above observations in Class$(\x_k)$},\\0 \quad\text{otherwise.}\end{cases}
\end{displaymath}
If there are only two classes, then we are back to the AUC or equivalently the WRS statistic.
%If $C=2$, then this formulation simply maximizes the WRS statistic, with the positive class as Class~1 and the negative class as Class~2.

\subsection{Experimental Results}

\begin{table}
\centering
\caption{Detailed experimental results on ROC Flexibility}
\scalebox{0.8}{
\begin{tabular}{cccccccccccccccc}
\hline\hline
& & & \multicolumn{10}{c}{Runs}& & \multicolumn{2}{c}{Statistics}\\
\cline{4-13}
\cline{15-16}
\multicolumn{3}{c}{Algorithm}& 1& 2& 3& 4& 5& 6& 7& 8& 9& 10& & Mean& Std. Dev.\\
\hline
\multirow{4}{*}{\rotatebox{90}{Baseline methods~~}}&
\multirow{2}{*}{LR}& train&29.12&29.65&31.02&32.34&31.11&30.93&32.27&33.91&32.84&28.95&&31.21&1.65\\
& & test&31.52&32.16&33.13&28.64&29.73&31.98&31.59&30.39&30.92&33.45&&31.35&1.48\\
\cline{2-16}
& \multirow{2}{*}{SVM}& train&29.19&29.74&30.83&32.15&31.27&30.84&32.26&32.34&32.81&27.97&&30.94&1.57\\
& & test&31.63&32.34&33.09&28.93&29.39&31.94&31.54&28.49&30.92&32.76&&31.10&1.63\\
\cline{2-16}
& \multirow{2}{*}{RB}& train&28.34&27.53&28.04&31.12&30.02&28.63&28.04&30.82&30.14&27.36&&29.00&1.39\\
& & test&30.39&30.91&29.98&27.37&28.89&29.69&29.50&27.84&28.90&32.23&&29.57&1.43\\
\cline{2-16}
& \multirow{2}{*}{P-norm Push}& train&29.12&29.64&30.85&32.46&31.37&30.93&32.27&33.75&32.81&30.10&&31.33&1.48\\
& & test&31.52&32.14&33.01&29.08&29.39&32.00&31.59&30.25&30.91&34.36&&31.43&1.61\\
\cline{1-16}
\multirow{4}{*}{\rotatebox{90}{MIO-based methods~}}&
\multirow{2}{*}{$K=50$}& train&31.35&30.60&31.02&33.74&32.78&30.93&32.27&33.91&32.84&30.20&&31.96&1.32\\
& & test&33.10&33.44&33.13&30.92&31.70&31.98&31.59&30.39&30.92&34.45&&32.16&1.31\\
\cline{2-16}
& \multirow{2}{*}{$K=100$}& train&31.44&30.65&30.37&33.82&32.52&30.93&31.71&33.88&32.84&30.28&&31.84&1.36\\
& & test&33.12&33.52&32.59&31.01&31.48&31.98&31.29&30.45&30.92&34.51&&32.09&1.31\\
\cline{2-16}
& \multirow{2}{*}{$K=150$}&train &30.63&30.69&31.06&32.78&32.13&31.27&32.27&32.95&33.00&29.78&&31.65&1.12\\
& & test&32.14&33.55&33.19&29.32&30.67&32.33&31.59&29.40&31.10&34.13&&31.74&1.65\\
\cline{2-16}
& \multirow{2}{*}{Full MIO}& train&26.92&27.15&27.04&28.14&30.90&30.99&26.41&30.12&26.96&29.69&&28.43&1.80\\
& & test&28.76&29.08&29.16&25.75&28.57&31.95&28.34&27.34&26.74&33.92&&28.96&2.40\\
\hline\hline
\end{tabular}
}
\end{table}

\begin{table}[htbp]
\centering
\caption{Detailed experimental results on Abalone19}
\scalebox{0.8}{
\begin{tabular}{cccccccccccccccc}
\hline\hline
& & & \multicolumn{10}{c}{Runs}& & \multicolumn{2}{c}{Statistics}\\
\cline{4-13}
\cline{15-16}
\multicolumn{3}{c}{Algorithm}& 1& 2& 3& 4& 5& 6& 7& 8& 9& 10& & Mean& Std. Dev.\\
\hline
\multirow{4}{*}{\rotatebox{90}{Baseline methods~~}}&
\multirow{2}{*}{LR}&train&3.50&3.56&3.92&3.69&3.56&3.94&2.66&4.19&3.92&3.35&&3.63&0.43\\
&&test&3.01&2.91&3.03&3.14&2.75&3.32&3.76&2.22&2.53&2.94&&2.96&0.42\\
\cline{2-16}
&\multirow{2}{*}{SVM}&train&3.46&3.43&3.46&3.39&3.51&2.99&2.46&4.23&3.85&3.37&&3.41&0.47\\
&&test&2.94&2.80&3.22&3.18&2.78&3.32&4.22&2.31&2.49&2.92&&3.02&0.53\\
\cline{2-16}
&\multirow{2}{*}{RB}&train&3.73&3.91&2.56&3.32&3.51&2.69&2.34&4.04&3.87&4.06&&3.40&0.65\\
&&test&2.47&2.49&2.76&2.89&2.60&3.08&3.66&1.92&2.13&2.62&&2.66&0.49\\
\cline{2-16}
&\multirow{2}{*}{P-norm Push}&train&3.44&3.52&3.74&3.40&3.45&3.04&2.44&4.20&3.80&3.33&&3.44&0.47\\
&&test&2.97&2.79&3.05&3.08&2.77&3.23&4.33&2.65&2.55&2.86&&3.03&0.50\\
\hline
\multirow{4}{*}{\rotatebox{90}{MIO-based methods~}}&
\multirow{2}{*}{$K=250$}&train&5.24&4.52&4.82&4.86&5.37&4.58&3.84&5.57&5.69&4.44&&4.89&0.58\\
&&test&2.88&3.25&2.99&3.33&3.15&3.22&3.41&2.37&2.26&3.91&&3.08&0.49\\
\cline{2-16}
&\multirow{2}{*}{$K=500$}&train&4.49&3.72&4.70&4.60&5.10&4.01&3.66&5.04&4.68&4.49&&4.45&0.50\\
&&test&2.85&2.87&2.89&3.12&2.50&3.11&3.36&2.20&2.70&3.26&&2.89&0.35\\
\cline{2-16}
&\multirow{2}{*}{$K=750$}&train&3.79&3.55&4.75&3.94&4.95&3.90&3.57&4.83&3.74&4.27&&4.13&0.53\\
&&test&2.64&2.96&2.80&2.90&2.61&3.25&3.14&1.92&2.47&2.96&&2.76&0.38\\
\cline{2-16}
&\multirow{2}{*}{Full MIO}&train&2.60&2.37&2.29&2.54&2.66&2.14&2.00&3.11&2.79&2.89&&2.54&0.35\\
&&test&2.22&2.47&2.47&2.98&2.07&2.78&3.24&1.66&1.99&2.32&&2.42&0.48\\
\hline\hline
\end{tabular}
}
\end{table}

\begin{table}[htbp]
\centering
\caption{Detailed experimental results on UIS}
\scalebox{0.8}{
\begin{tabular}{cccccccccccccccc}
\hline\hline
& & & \multicolumn{10}{c}{Runs}& & \multicolumn{2}{c}{Statistics}\\
\cline{4-13}
\cline{15-16}
\multicolumn{3}{c}{Algorithm}& 1& 2& 3& 4& 5& 6& 7& 8& 9& 10& & Mean& Std. Dev.\\
\hline
\multirow{4}{*}{\rotatebox{90}{Baseline methods~~}}&
\multirow{2}{*}{LR}& train&18.94&18.64&17.99&19.21&18.50&17.69&21.46&19.38&20.10&16.69&&18.86&1.32\\
& &test&17.11&17.84&17.11&16.87&18.72&17.99&17.03&18.15&17.40&20.57&&17.88&1.11\\
\cline{2-16}
&\multirow{2}{*}{SVM}&train&18.37&17.83&17.67&18.39&18.20&17.64&21.17&19.13&20.00&16.18&&18.46&1.38\\
& &test&16.91&17.62&17.32&16.78&18.73&17.71&16.91&18.12&17.18&20.80&&17.81&1.21\\
\cline{2-16}
&\multirow{2}{*}{RB}&train&20.40&19.13&19.59&20.40&18.46&18.31&21.97&19.30&20.10&16.69&&19.44&1.44\\
& &test&16.18&17.77&18.33&16.30&19.88&17.81&15.64&18.00&17.39&19.67&&17.70&1.40\\
\cline{2-16}
&\multirow{2}{*}{P-norm Push}&train&18.93&17.92&17.93&19.18&18.49&17.68&21.55&19.42&20.10&16.60&&18.78&1.40\\
& &test&17.09&17.82&17.19&16.86&18.73&17.96&17.04&18.22&17.42&20.61&&17.89&1.13\\
\cline{1-16}
\multirow{4}{*}{\rotatebox{90}{MIO-based methods~}}&
\multirow{2}{*}{$K=50$}&train&20.53&20.23&20.52&21.13&19.34&20.32&22.43&20.90&21.25&17.81&&20.45&1.23\\
& &test&19.40&19.00&16.57&17.94&18.65&19.73&17.15&16.56&16.14&18.87&&18.00&1.31\\
\cline{2-16}
&\multirow{2}{*}{$K=100$}&train&20.70&19.66&19.92&20.20&18.33&19.64&21.78&20.04&20.26&17.11&&19.76&1.27\\
& &test&18.94&18.69&18.06&18.00&20.60&19.48&15.42&18.93&17.67&20.58&&18.64&1.51\\
\cline{2-16}
&\multirow{2}{*}{$K=150$}&train&19.91&18.97&19.37&19.53&18.02&19.79&20.32&19.80&20.01&16.90&&19.26&1.05\\
& &test&17.59&17.24&18.68&17.11&18.47&17.86&15.64&17.01&16.34&22.00&&17.79&1.73\\
\cline{2-16}
&\multirow{2}{*}{Full MIO}&train&19.06&18.99&18.10&19.90&17.39&18.89&21.68&19.03&19.11&16.24&&18.84&1.44\\
& &test&18.47&17.07&18.08&18.03&18.42&18.42&17.02&17.38&17.16&18.82&&17.89&0.67\\
\hline\hline
\end{tabular}
}
\end{table}

\begin{table}[htbp]
\centering
\caption{Detailed experimental results on Travel}
\scalebox{0.8}{
\begin{tabular}{cccccccccccccccc}
\hline\hline
& & & \multicolumn{10}{c}{Runs}& & \multicolumn{2}{c}{Statistics}\\
\cline{4-13}
\cline{15-16}
\multicolumn{3}{c}{Algorithm}& 1& 2& 3& 4& 5& 6& 7& 8& 9& 10& & Mean& Std. Dev.\\
\hline
\multirow{4}{*}{\rotatebox{90}{Baseline methods~~}}&
\multirow{2}{*}{LR}&train&31.48&26.63&27.22&28.60&29.25&28.44&27.94&28.70&25.55&27.81&&28.16&1.60\\
& &test&23.94&28.69&28.34&26.98&26.76&25.57&27.88&26.95&30.02&28.03&&27.32&1.70\\
\cline{2-16}
&\multirow{2}{*}{SVM}&train&31.08&26.13&26.75&28.23&28.58&27.17&27.29&28.26&25.12&27.29&&27.59&1.61\\
& &test&23.41&27.74&27.84&26.93&26.37&24.56&27.22&26.58&29.65&27.79&&26.81&1.76\\
\cline{2-16}
&\multirow{2}{*}{RB}&train&29.83&25.00&25.31&27.01&27.70&27.28&26.72&26.63&24.09&26.20&&26.57&1.60\\
& &test&21.93&26.46&26.35&24.33&23.36&24.00&25.20&25.01&27.48&25.40&&24.95&1.63\\
\cline{2-16}
&\multirow{2}{*}{P-norm Push}&train&31.52&26.62&27.12&28.49&29.26&28.37&27.93&28.49&25.52&27.61&&28.09&1.62\\
& &test&24.05&28.39&28.25&26.80&26.78&25.51&27.86&26.74&30.00&28.01&&27.24&1.66\\
\hline
\multirow{4}{*}{\rotatebox{90}{MIO-based methods~}}&
\multirow{2}{*}{$K=50$}&train&31.69&26.91&27.18&28.73&29.31&28.74&28.19&28.83&25.58&27.87&&28.30&1.63\\
& &test&24.26&29.20&28.48&27.05&26.88&26.42&28.15&26.97&30.46&28.24&&27.61&1.70\\
\cline{2-16}
&\multirow{2}{*}{$K=100$}&train&31.19&27.05&27.01&28.81&29.62&28.75&27.65&28.53&25.54&28.23&&28.24&1.56\\
& &test&24.15&29.63&28.46&27.01&26.31&27.15&26.55&27.48&29.39&27.81&&27.39&1.60\\
\cline{2-16}
&\multirow{2}{*}{$K=150$}&train&28.76&26.28&25.39&27.95&29.24&27.50&26.90&28.28&24.99&25.93&&27.12&1.45\\
& &test&20.99&28.16&26.59&26.53&26.21&23.86&25.99&26.78&29.19&25.68&&26.00&2.26\\
\cline{2-16}
&\multirow{2}{*}{Full MIO}&train&29.57&26.14&26.84&26.55&28.62&26.38&25.85&27.93&25.08&26.44&&26.94&1.36\\
& &test&23.16&27.82&27.61&25.01&26.30&24.42&26.41&26.01&29.55&26.85&&26.31&1.83\\
\hline\hline
\end{tabular}
}
\end{table}

\begin{table}[htbp]
\centering
\caption{Detailed experimental results on NHANES}
\scalebox{0.8}{
\begin{tabular}{cccccccccccccccc}
\hline\hline
& & & \multicolumn{10}{c}{Runs}& & \multicolumn{2}{c}{Statistics}\\
\cline{4-13}
\cline{15-16}
\multicolumn{3}{c}{Algorithm}& 1& 2& 3& 4& 5& 6& 7& 8& 9& 10& & Mean& Std. Dev.\\
\hline
\multirow{4}{*}{\rotatebox{90}{Baseline methods~~}}&
\multirow{2}{*}{LR}&train&14.50&12.76&16.84&13.89&17.18&15.84&14.62&13.89&15.25&12.14&&14.69&1.63\\
& &test&14.43&14.16&10.31&14.63&10.33&11.78&14.74&13.42&12.39&14.40&&13.06&1.74\\
\cline{2-16}
&\multirow{2}{*}{SVM}&train&13.04&11.42&15.80&12.64&16.63&15.72&13.57&13.17&14.99&11.31&&13.83&1.87\\
& &test&14.33&14.08&10.27&14.63&10.13&11.75&14.86&13.33&12.12&14.27&&12.98&1.79\\
\cline{2-16}
&\multirow{2}{*}{RB}&train&13.05&12.05&16.73&13.35&16.84&15.20&12.93&13.03&13.87&10.40&&13.75&2.01\\
& &test&11.95&13.68&9.78&13.06&9.31&10.67&12.49&12.43&12.67&14.98&&12.10&1.75\\
\cline{2-16}
&\multirow{2}{*}{P-norm Push}&train&14.43&12.53&16.63&13.53&16.12&15.89&14.49&13.74&15.32&11.87&&14.46&1.57\\
& &test&14.28&14.39&10.10&14.76&10.24&11.85&14.88&13.63&13.30&14.41&&13.18&1.82\\
\hline
\multirow{4}{*}{\rotatebox{90}{MIO-based methods~}}&
\multirow{2}{*}{$K=50$}&train&15.32&13.04&17.90&14.96&17.84&17.02&14.95&14.64&15.77&13.35&&15.48&1.69\\
& &test&14.54&14.39&10.38&13.37&11.07&13.33&14.65&12.99&13.11&14.74&&13.26&1.50\\
\cline{2-16}
&\multirow{2}{*}{$K=100$}&train&14.11&11.83&17.52&13.77&18.28&16.48&14.49&13.75&15.04&14.95&&15.02&1.93\\
& &test&13.58&14.25&10.29&12.76&9.83&12.96&13.53&12.27&14.95&12.74&&12.71&1.61\\
\cline{2-16}
&\multirow{2}{*}{$K=150$}&train&14.30&12.49&16.91&14.87&16.40&16.56&14.83&13.89&14.77&12.27&&14.73&1.59\\
& &test&13.83&14.30&10.73&12.75&10.44&11.91&14.86&13.75&12.40&14.38&&12.94&1.55\\
\cline{2-16}
&\multirow{2}{*}{Full MIO}&train&13.36&12.11&15.79&13.48&15.66&14.48&14.17&13.25&14.13&12.21&&13.87&1.25\\
& &test&14.46&14.75&10.42&14.83&10.19&11.76&14.53&13.44&12.19&14.37&&13.09&1.82\\
\hline\hline
\end{tabular}
}
\end{table}

\begin{table}[htbp]
\centering
\caption{Detailed experimental results on Pima}
\scalebox{0.8}{
\begin{tabular}{cccccccccccccccc}
\hline\hline
& & & \multicolumn{10}{c}{Runs}& & \multicolumn{2}{c}{Statistics}\\
\cline{4-13}
\cline{15-16}
\multicolumn{3}{c}{Algorithm}& 1& 2& 3& 4& 5& 6& 7& 8& 9& 10& & Mean& Std. Dev.\\
\hline
\multirow{4}{*}{\rotatebox{90}{Baseline methods~~}}&
\multirow{2}{*}{LR}&train&37.27&38.05&33.65&33.84&34.54&33.02&36.10&36.57&35.96&35.98&&35.50&1.66\\
&&test&32.64&31.77&36.70&36.07&34.88&36.47&33.85&32.22&34.22&32.94&&34.18&1.81\\
\cline{2-16}
&\multirow{2}{*}{SVM}&train&36.69&37.74&33.52&33.59&34.29&32.93&36.08&36.53&35.88&35.71&&35.30&1.61\\
&&test&31.91&31.93&36.73&35.66&34.64&36.41&33.90&32.19&34.24&32.77&&34.04&1.83\\
\cline{2-16}
&\multirow{2}{*}{RB}&train&37.65&36.95&34.29&34.36&35.12&33.27&36.56&36.96&36.29&36.57&&35.80&1.44\\
&&test&32.23&31.17&35.97&36.14&34.26&35.40&33.37&33.17&34.02&32.53&&33.83&1.65\\
\cline{2-16}
&\multirow{2}{*}{P-norm Push}&train&37.39&38.25&33.74&33.99&34.79&33.11&36.20&36.62&36.10&36.26&&35.64&1.67\\
&&test&32.64&31.70&36.67&36.04&35.10&36.65&33.76&32.21&34.20&33.42&&34.24&1.82\\
\hline
\multirow{4}{*}{\rotatebox{90}{MIO-based methods~}}&
\multirow{2}{*}{$K=50$}&train&37.41&38.36&33.74&34.17&34.56&33.30&36.52&36.71&36.31&36.39&&35.75&1.67\\
&&test&32.70&32.18&36.28&36.36&34.91&37.37&33.07&33.86&34.50&33.19&&34.44&1.76\\
\cline{2-16}
&\multirow{2}{*}{$K=100$}&train&37.58&37.99&33.54&33.67&34.75&33.10&35.10&36.30&36.22&36.07&&35.43&1.69\\
&&test&31.41&30.77&33.82&35.93&35.42&35.57&32.34&33.97&33.03&33.31&&33.56&1.87\\
\cline{2-16}
&\multirow{2}{*}{$K=150$}&train&37.25&37.97&33.00&33.53&34.14&31.35&35.35&36.40&35.54&33.99&&34.85&1.96\\
&&test&32.50&31.75&35.93&36.03&35.62&37.05&32.22&32.38&32.98&30.74&&33.72&2.21\\
\cline{2-16}
&\multirow{2}{*}{Full MIO}&train&37.16&37.69&33.78&33.98&34.07&31.99&35.14&36.21&31.71&35.95&&34.77&2.03\\
&&test&32.48&31.85&36.06&35.60&35.02&35.19&33.40&32.50&31.10&33.35&&33.65&2.01\\
\hline\hline
\end{tabular}
}
\end{table}

\begin{table}[htbp]
\centering
\caption{Detailed experimental results on the Gaussians data set}
\scalebox{0.8}{
\begin{tabular}{cccccccccccccccc}
\hline\hline
& & & \multicolumn{10}{c}{Runs}& & \multicolumn{2}{c}{Statistics}\\
\cline{4-13}
\cline{15-16}
\multicolumn{3}{c}{Algorithm}& 1& 2& 3& 4& 5& 6& 7& 8& 9& 10& & Mean& Std. Dev.\\
\hline
\multirow{4}{*}{\rotatebox{90}{Baseline methods~~}}&
\multirow{2}{*}{LR}&train&69.23&68.26&67.17&68.13&75.38&72.75&68.76&67.18&67.24&68.39&&69.25&2.70\\
&&test&64.89&65.62&66.88&65.49&59.25&61.51&64.61&66.58&66.65&65.41&&64.89&2.45\\
\cline{2-16}
&\multirow{2}{*}{SVM}&train&69.30 &68.27 &67.15 &68.19 &75.39 &72.74 &68.90 &67.18 &67.24 &68.38 &&69.28 &2.70\\
&&test&64.92 &65.64 &66.90 &65.59 &59.26 &61.52 &64.81 &66.58 &66.65 &65.41 &&64.73 &2.45\\
\cline{2-16}
&\multirow{2}{*}{RB}&train&71.14&70.53&69.43&70.51&76.05&74.28&70.21&70.04&69.70&71.22&&71.31&2.15\\
&&test&67.58&68.68&68.36&67.81&62.49&64.78&66.91&68.40&69.40&69.36&&67.13&2.06\\
\cline{2-16}
&\multirow{2}{*}{P-norm Push}&train&69.22&68.21&67.19&68.12&75.39&72.75&68.69&67.18&67.26&68.38&&69.24&2.70\\
&&test&64.86&65.45&66.86&65.46&59.26&61.51&64.47&66.57&66.65&65.38&&64.65&2.43\\
\hline
\multirow{4}{*}{\rotatebox{90}{MIO-based methods~}}&
\multirow{2}{*}{$K=50$}&train&71.32&70.56&69.36&71.28&76.57&74.59&71.69&70.16&70.03&71.61&&71.72&2.22\\
&&test&68.57&69.10&70.38&68.45&62.99&65.26&67.82&69.73&69.81&68.22&&68.03&2.27\\
\cline{2-16}
&\multirow{2}{*}{$K=100$}&train&71.27&70.43&69.61&71.39&76.56&74.59&71.61&70.33&70.06&71.81&&71.76&2.18\\
&&test&68.48&69.12&70.25&68.26&63.05&65.11&67.58&69.81&69.82&67.64&&67.91&2.27\\
\cline{2-16}
&\multirow{2}{*}{$K=150$}&train&71.01&70.55&69.08&71.41&76.45&74.54&71.76&69.95&69.41&71.67&&71.58&2.30\\
&&test&67.31&69.37&70.34&68.46&62.55&65.41&68.01&68.38&69.92&68.11&&67.79&2.30\\
\cline{2-16}
&\multirow{2}{*}{Full MIO}&train&62.81&63.50&61.76&65.16&70.67&66.56&62.53&64.68&64.62&64.69&&64.70&2.53\\
&&test&61.28&61.58&62.34&60.50&55.83&55.89&59.76&59.87&61.02&60.81&&59.89&2.26\\
\hline\hline
\end{tabular}
}
\end{table}

\begin{table}[htbp]
\centering
\caption{Detailed experimental results on Haberman Survival}
\scalebox{0.8}{
\begin{tabular}{cccccccccccccccc}
\hline\hline
& & & \multicolumn{10}{c}{Runs}& & \multicolumn{2}{c}{Statistics}\\
\cline{4-13}
\cline{15-16}
\multicolumn{3}{c}{Algorithm}& 1& 2& 3& 4& 5& 6& 7& 8& 9& 10& & Mean& Std. Dev.\\
\hline
\multirow{4}{*}{\rotatebox{90}{Baseline methods~~}}&
\multirow{2}{*}{LR}&train&13.45&11.92&14.94&13.98&13.18&13.08&12.47&12.88&12.28&11.17&&12.94&1.07\\
&&test&11.45&13.79&10.69&12.01&13.02&13.13&13.47&12.84&13.24&14.58&&12.82&1.15\\
\cline{2-16}
&\multirow{2}{*}{SVM}&train&13.43&11.96&15.02&13.83&13.16&13.19&12.52&12.82&12.34&11.18&&12.95&1.06\\
&&test&11.42&13.79&10.70&11.52&13.08&13.06&12.78&12.82&12.58&14.56&&12.63&1.15\\
\cline{2-16}
&\multirow{2}{*}{RB}&train&14.35&13.55&16.48&13.83&13.84&13.81&13.95&14.47&13.89&11.33&&13.95&1.24\\
&&test&11.29&13.08&9.60&11.85&12.13&12.62&12.39&11.33&12.74&13.01&&12.01&1.06\\
\cline{2-16}
&\multirow{2}{*}{P-norm Push}&train&13.44&11.91&14.88&14.00&13.18&13.11&12.50&12.92&12.26&11.17&&12.94&1.06\\
&&test&11.51&13.66&10.72&11.98&12.94&13.14&12.84&12.80&12.24&14.55&&12.64&1.09\\
\hline
\multirow{3}{*}{\rotatebox{90}{MIO-based~~~}}&
\multirow{2}{*}{$K=50$}&train&13.68&11.91&14.92&14.12&12.95&13.00&12.54&14.39&12.49&11.02&&13.10&1.19\\
&&test&10.10&13.82&10.62&11.99&12.92&13.07&13.48&11.88&13.93&14.60&&12.64&1.47\\
\cline{2-16}
&\multirow{2}{*}{$K=100$}&train&14.16&12.13&14.94&14.45&12.98&13.03&11.87&12.85&12.63&11.14&&13.02&1.20\\
&&test&11.65&13.05&10.52&12.24&12.95&13.19&13.67&12.81&13.23&14.66&&12.80&1.13\\
\cline{2-16}
&\multirow{2}{*}{Full MIO}&train&13.90&12.45&14.92&13.92&12.91&13.26&12.53&12.82&12.97&11.63&&13.13&0.92\\
&&test&11.62&13.26&10.60&11.44&12.62&12.82&12.47&12.76&12.61&14.42&&12.46&1.05\\
\hline\hline
\end{tabular}
}
\end{table}

\begin{table}
\centering
\caption{Detailed experimental results on Polypharm}
\scalebox{0.8}{
\begin{tabular}{cccccccccccccccc}
\hline\hline
& & & \multicolumn{10}{c}{Runs}& & \multicolumn{2}{c}{Statistics}\\
\cline{4-13}
\cline{15-16}
\multicolumn{3}{c}{Algorithm}& 1& 2& 3& 4& 5& 6& 7& 8& 9& 10& & Mean& Std. Dev.\\
\hline
\multirow{4}{*}{\rotatebox{90}{Baseline methods~~}}&
\multirow{2}{*}{LR}&train&17.69&20.69&20.52&19.93&18.63&19.69&21.49&17.99&17.35&20.27&&19.43&1.42\\
&&test&18.60&15.65&15.84&16.40&18.60&17.41&14.74&17.82&20.12&17.14&&17.23&1.63\\
\cline{2-16}
&\multirow{2}{*}{SVM}&train&17.57&20.69&20.15&19.85&18.24&19.60&21.46&17.48&17.06&20.26&&19.24&1.53\\
&&test&18.51&16.84&17.54&16.93&18.47&17.34&14.67&19.63&20.11&17.09&&17.71&1.56\\
\cline{2-16}
&\multirow{2}{*}{RB}&train&17.65&20.52&20.10&19.97&17.81&19.75&21.03&16.90&17.15&20.25&&19.11&1.55\\
&&test&18.98&15.59&16.99&16.10&18.51&17.01&14.89&19.35&19.90&16.67&&17.40&1.70\\
\cline{2-16}
&\multirow{2}{*}{P-norm Push}&train&17.30&19.43&19.80&18.71&16.76&19.30&20.53&19.34&18.98&17.16&&18.73&1.25\\
&&test&20.36&17.34&17.75&18.87&19.33&16.96&15.66&17.51&17.95&18.79&&18.05&1.33\\
\hline
\multirow{4}{*}{\rotatebox{90}{MIO-based methods~}}&
\multirow{2}{*}{$K=50$}&train&17.71&19.64&19.95&18.91&19.65&19.48&20.64&19.05&19.37&17.68&&19.21&0.93\\
&&test&20.25&16.30&17.02&16.71&18.18&17.58&15.47&17.55&16.49&18.20&&17.37&1.33\\
\cline{2-16}
&\multirow{2}{*}{$K=100$}&train&17.09&19.29&19.60&18.41&16.89&19.40&20.64&18.07&19.54&17.38&&18.63&1.25\\
&&test&19.88&16.82&17.53&18.39&16.96&16.53&15.69&17.40&17.57&19.13&&17.59&1.25\\
\cline{2-16}
&\multirow{2}{*}{$K=150$}&train&17.11 &17.78 &17.75 &17.77 &16.43 &19.21 &18.99 &18.76 &19.54 &17.05 & &18.04 &1.04  \\
&&test&19.93 &16.15 &13.85 &17.70 &17.55 &17.97 &14.21 &17.81 &16.84 &19.54 & &17.16 &1.99  \\
\cline{2-16}
&\multirow{2}{*}{Full MIO}&train&16.95&18.71&19.40&17.21&16.94&18.35&20.33&18.48&19.26&15.92&&18.16&1.36\\
&&test&20.05&16.50&17.41&17.68&18.73&14.62&15.84&16.92&16.16&18.28&&17.22&1.57\\
\hline\hline
\end{tabular}
}
\end{table}

\begin{table}[htbp]
\centering
\caption{Detailed experimental results on Glow500}
\scalebox{0.8}{
\begin{tabular}{cccccccccccccccc}
\hline\hline
& & & \multicolumn{10}{c}{Runs}& & \multicolumn{2}{c}{Statistics}\\
\cline{4-13}
\cline{15-16}
\multicolumn{3}{c}{Algorithm}& 1& 2& 3& 4& 5& 6& 7& 8& 9& 10& & Mean& Std. Dev.\\
\hline
\multirow{4}{*}{\rotatebox{90}{Baseline methods~~}}&
\multirow{2}{*}{LR}&train&14.66&17.70&18.65&17.44&18.44&17.84&16.95&16.51&16.57&17.87&&17.26&1.16\\
&&test&18.32&15.45&16.06&17.11&14.78&17.68&17.64&16.90&17.45&15.43&&16.68&1.18\\
\cline{2-16}
&\multirow{2}{*}{SVM}&train&13.94&17.13&18.67&16.92&17.57&17.80&16.18&15.41&15.07&16.52&&16.52&1.41\\
&&test&18.48&15.83&15.59&17.11&16.31&17.63&18.21&17.90&18.91&16.59&&17.26&1.15\\
\cline{2-16}
&\multirow{2}{*}{RB}&train&14.34&17.77&18.65&18.12&17.91&17.22&17.94&16.31&16.42&17.65&&17.23&1.25\\
&&test&18.48&16.44&16.77&16.29&16.32&17.71&16.83&17.37&17.20&16.67&&17.01&0.69\\
\cline{2-16}
&\multirow{2}{*}{P-norm Push}&train&14.10&17.67&18.96&17.36&18.52&17.97&16.56&15.66&16.19&17.22&&17.02&1.44\\
&&test&18.35&15.92&15.68&17.56&15.41&18.07&18.65&18.01&18.45&16.28&&17.24&1.27\\
\hline
\multirow{4}{*}{\rotatebox{90}{MIO-based methods~}}&
\multirow{2}{*}{$K=50$}&train&15.62&18.85&19.20&19.02&19.21&18.65&18.30&17.31&16.74&19.12&&18.20&1.24\\
&&test&16.89&15.23&17.18&16.43&17.03&18.33&17.65&17.30&17.15&14.68&&16.79&1.09\\
\cline{2-16}
&\multirow{2}{*}{$K=100$}&train&15.30&18.27&18.51&18.32&19.42&18.30&17.86&16.18&16.93&17.89&&17.70&1.22\\
&&test&17.38&16.82&15.39&17.69&16.18&18.50&19.06&15.74&19.68&14.59&&17.10&1.66\\
\cline{2-16}
&\multirow{2}{*}{$K=150$}&train&14.53&17.85&18.69&17.90&18.81&17.25&16.87&15.85&15.53&17.53&&17.08&1.40\\
&&test&19.26&14.85&15.58&16.73&16.07&17.40&17.27&18.28&16.70&15.65&&16.78&1.33\\
\cline{2-16}
&\multirow{2}{*}{Full MIO}&train&13.51&17.51&17.46&16.46&18.40&16.97&16.71&16.35&16.64&16.94&&16.69&1.28\\
&&test&16.47&15.59&15.74&15.78&15.54&17.26&15.72&18.12&19.36&15.22&&16.48&1.35\\
\hline\hline
\end{tabular}
}
\end{table}

\begin{table}[htbp]
\centering
\caption{Detailed experimental results on ROC Flexibility, with different $c$ and $\varepsilon$ values.}
\scalebox{0.8}{
\begin{tabular}{cccccccccccccccc}
\hline\hline
& & & \multicolumn{10}{c}{Runs}& & \multicolumn{2}{c}{Statistics}\\
\cline{4-13}
\cline{15-16}
\multicolumn{3}{c}{Parameter values}& 1& 2& 3& 4& 5& 6& 7& 8& 9& 10& & Mean& Std. Dev.\\
\hline
\multirow{12}{*}{\rotatebox{90}{$K=100, \varepsilon=10^{-4}$~}}&
\multirow{2}{*}{$c=10^{-1}$}&train&29.12&29.65&31.02&32.34&31.11&30.93&32.27&33.91&32.84&29.95&&31.31&1.52\\
&&test&31.52&32.16&33.13&28.64&29.73&31.98&31.59&30.39&30.92&33.45&&31.35&1.48\\
\cline{2-16}
&\multirow{2}{*}{$c=10^{-2}$}&train&30.35&30.18&27.91&33.80&31.82&30.93&32.27&32.71&32.84&30.28&&31.31&1.72\\
&&test&31.78&32.42&29.64&30.97&30.35&31.98&31.59&28.83&30.92&34.51&&31.30&1.57\\
\cline{2-16}
&\multirow{2}{*}{$c=10^{-3}$}&train&31.44&30.65&30.37&33.82&32.52&30.93&31.71&33.88&32.84&30.28&&31.84&1.36\\
&&test&33.12&33.52&32.59&31.01&31.48&31.98&31.29&30.45&30.92&34.51&&32.09&1.31\\
\cline{2-16}
&\multirow{2}{*}{$c=10^{-4}$}&train&31.50&30.69&31.06&33.14&32.74&30.93&32.27&33.93&32.84&30.35&&31.94&1.20\\
&&test&33.17&33.55&33.19&29.59&31.65&31.98&31.59&30.48&30.92&34.52&&32.06&1.53\\
\cline{2-16}
&\multirow{2}{*}{$c=10^{-5}$}&train&31.50&30.69&31.06&33.86&32.74&30.93&32.27&33.93&32.84&30.35&&32.02&1.30\\
&&test&33.17&33.55&33.19&31.08&31.65&31.98&31.59&30.48&30.92&34.52&&32.21&1.32\\
\cline{2-16}
&\multirow{2}{*}{$c=10^{-6}$}&train&30.63 &30.60 &31.06 &33.86 &32.13 &30.93 &32.27 &33.93 &32.84 &30.35 & &31.86 &1.35 \\
&&test&32.14 &33.55 &33.19 &31.08 &30.67 &31.98 &31.59 &30.48 &30.92 &34.52 & &32.01 &1.35 \\
\hline
\multirow{12}{*}{\rotatebox{90}{$K=100, c=10^{-3}$~}}&
\multirow{2}{*}{$\varepsilon=10^{-1}$}&train&31.45 &30.69 &31.06 &33.05 &30.08 &30.93 &32.27 &33.65 &32.84 &30.24 & &31.62 &1.25 \\
&&test&32.12 &33.55 &33.19 &29.21 &28.21 &31.98 &31.59 &29.62 &30.92 &34.47 & &31.59 &2.07 \\
\cline{2-16}
&\multirow{2}{*}{$\varepsilon=10^{-2}$}&train&30.63 &30.69 &31.06 &33.05 &32.74 &30.93 &32.27 &33.92 &32.84 &30.35 & &31.85 &1.26 \\
&&test&32.14 &33.55 &33.19 &29.11 &31.65 &31.98 &31.59 &30.49 &30.92 &34.52 & &31.91 &1.57 \\
\cline{2-16}
&\multirow{2}{*}{$\varepsilon=10^{-3}$}&train&30.63&30.69&31.06&33.84&32.74&30.93&32.27&33.93&32.84&30.35&&31.93&1.36\\
&&test&32.14&33.55&33.19&31.03&31.65&31.98&31.59&30.48&30.92&34.52&&32.10&1.28\\
\cline{2-16}
&\multirow{2}{*}{$\varepsilon=10^{-4}$}&train&31.44&30.65&30.37&33.82&32.52&30.93&31.71&33.88&32.84&30.28&&31.84&1.36\\
&&test&33.12&33.52&32.59&31.01&31.48&31.98&31.29&30.45&30.92&34.51&&32.09&1.31\\
\cline{2-16}
&\multirow{2}{*}{$\varepsilon=10^{-5}$}&train&31.50&30.69&31.06&33.86&32.74&30.93&32.27&33.93&32.84&30.35&&32.02&1.30\\
&&test&33.17&33.55&33.19&31.08&31.65&31.98&31.59&30.48&30.92&34.52&&32.21&1.32\\
\cline{2-16}
&\multirow{2}{*}{$\varepsilon=10^{-6}$}&train&31.38&30.19&31.06&33.74&31.55&30.93&31.93&33.90&32.84&28.27&&31.58&1.68\\
&&test&33.07&33.04&33.19&30.92&29.61&31.98&31.48&30.41&30.92&32.75&&31.74&1.26\\
\hline\hline
\end{tabular}
}
\end{table}

\begin{table}[htbp]
\centering
\caption{Detailed experimental results on UIS, with different $c$ and $\varepsilon$ values.}
\scalebox{0.8}{
\begin{tabular}{cccccccccccccccc}
\hline\hline
& & & \multicolumn{10}{c}{Runs}& & \multicolumn{2}{c}{Statistics}\\
\cline{4-13}
\cline{15-16}
\multicolumn{3}{c}{Parameter values}& 1& 2& 3& 4& 5& 6& 7& 8& 9& 10& & Mean& Std. Dev.\\
\hline
\multirow{12}{*}{\rotatebox{90}{$K=100, \varepsilon=10^{-4}$~}}&
\multirow{2}{*}{$c=10^{-1}$}&train&18.94&18.64&17.99&19.21&18.50&20.25&21.46&19.38&20.10&17.08&&19.15&1.24\\
&&test&17.11&17.84&17.11&16.87&18.72&18.72&17.03&18.15&17.40&20.44&&17.94&1.11\\
\cline{2-16}
&\multirow{2}{*}{$c=10^{-2}$}&train&19.39&18.86&19.75&19.85&18.56&19.45&21.40&19.32&19.63&17.39&&19.36&1.02\\
&&test&18.21&18.39&17.14&17.84&20.71&16.78&16.05&18.44&17.20&20.70&&18.15&1.54\\
\cline{2-16}
&\multirow{2}{*}{$c=10^{-3}$}&train&20.17&19.92&19.47&19.80&18.81&20.16&21.46&19.82&19.71&17.53&&19.69&1.01\\
&&test&18.44&18.30&17.43&17.23&19.83&18.73&16.15&16.09&16.31&20.67&&17.92&1.57\\
\cline{2-16}
&\multirow{2}{*}{$c=10^{-4}$}&train&20.70&19.66&19.92&20.20&18.33&19.64&21.78&20.04&20.26&17.11&&19.76&1.27\\
&&test&18.94&18.69&18.06&18.00&20.60&19.48&15.42&18.93&17.67&20.58&&18.64&1.51\\
\cline{2-16}
&\multirow{2}{*}{$c=10^{-5}$}&train&20.55&19.03&19.73&20.83&19.00&19.46&22.16&20.02&20.48&17.61&&19.89&1.24\\
&&test&18.84&18.71&16.98&17.49&18.53&18.17&16.67&17.21&16.73&20.12&&17.94&1.12\\
\cline{2-16}
&\multirow{2}{*}{$c=10^{-6}$}&train&20.67&19.95&19.51&19.83&18.88&19.22&21.70&19.76&19.70&17.55&&19.68&1.08\\
&&test&18.41&17.52&16.80&17.90&18.30&19.93&15.38&15.92&16.25&20.96&&17.74&1.76\\
\hline
\multirow{12}{*}{\rotatebox{90}{$K=100, c=10^{-4}$~}}&
\multirow{2}{*}{$\varepsilon=10^{-1}$}&train&20.22 &19.26 &19.37 &20.45 &18.86 &19.90 &21.87 &19.88 &20.08 &17.10 & &19.70 &1.23 \\
&&test&18.78 &17.65 &17.07 &17.80 &18.78 &19.71 &17.22 &18.38 &18.11 &20.53 & &18.40 &1.09 \\
\cline{2-16}
&\multirow{2}{*}{$\varepsilon=10^{-2}$}&train&20.22&19.61&19.64&20.43&18.15&19.54&22.17&19.88&20.04&17.61&&19.73&1.24\\
&&test&18.28&18.56&18.45&17.39&19.47&17.11&16.71&18.42&16.46&19.44&&18.03&1.06\\
\cline{2-16}
&\multirow{2}{*}{$\varepsilon=10^{-3}$}&train&20.97&19.30&20.06&19.65&18.61&20.07&21.35&20.10&20.26&17.63&&19.80&1.09\\
&&test&17.98&18.00&18.01&17.46&20.04&18.67&16.45&17.91&18.26&20.60&&18.34&1.20\\
\cline{2-16}
&\multirow{2}{*}{$\varepsilon=10^{-4}$}&train&20.70&19.66&19.92&20.20&18.33&19.64&21.78&20.04&20.26&17.11&&19.76&1.27\\
&&test&18.94&18.69&18.06&18.00&20.60&19.48&15.42&18.93&17.67&20.58&&18.64&1.51\\
\cline{2-16}
&\multirow{2}{*}{$\varepsilon=10^{-5}$}&train&20.43&19.81&19.53&19.98&18.47&20.09&21.55&19.95&20.32&17.20&&19.73&1.17\\
&&test&17.06&18.74&16.70&17.95&19.08&19.08&15.03&17.83&16.45&20.86&&17.88&1.66\\
\cline{2-16}
&\multirow{2}{*}{$\varepsilon=10^{-6}$}&train&20.16&18.37&19.53&18.15&18.50&18.81&21.46&19.38&20.10&16.32&&19.08&1.40\\
&&test&18.23&18.24&17.66&18.23&18.72&19.51&17.03&18.15&17.40&19.12&&18.23&0.75\\
\hline\hline
\end{tabular}
}
\end{table}

\end{document}